\documentclass[journal]{IEEEtran}

\usepackage[colorlinks,linkcolor=blue,citecolor=blue]{hyperref}
\usepackage{bm}
\usepackage{bbm}
\usepackage{booktabs}
\usepackage{diagbox, slashbox}
\usepackage[vlined,ruled,linesnumbered]{algorithm2e}
\SetAlCapFnt{\small}        
\SetAlCapNameFnt{\small}    
\usepackage[pdftex]{graphicx}
\usepackage[caption=false,margin=0pt]{subfig}
\usepackage{cite}
\usepackage{amsthm,amsmath,amsfonts,amssymb,mathtools,url, multirow, enumerate,setspace,colortbl,stfloats, float}
\usepackage[dvipsnames]{xcolor}
\usepackage{xspace}
\xspaceaddexceptions{$} 
\usepackage{newtxtt}
\usepackage{sourcesanspro}
\usepackage{arydshln} 
\usepackage{makecell} %
\usepackage{enumitem}

\usepackage{array}

\usepackage{textcomp}
\usepackage{verbatim}

\usepackage{pifont}

\usepackage{threeparttable}

\usepackage{tcolorbox}

\usepackage{utfsym}
\newcommand{\cmark}{\textcolor{green!60!black}{\usym{2713}}} 
\newcommand{\xmark}{\textcolor{red!70!black}{\usym{2717}}}  

\theoremstyle{plain} 
\newtheorem{theorem}{Theorem}

\theoremstyle{definition} 

\newtheorem{assumption}{Assumption}

\theoremstyle{remark} 
\newtheorem*{remark}{Remark} 

\newcolumntype{C}{>{\centering\arraybackslash}p{1cm}}
\newcolumntype{R}{>{\raggedleft\arraybackslash}p{1cm}}

\newcommand{\knn}{\texorpdfstring{$k$NN}{kNN}\xspace}

\newcommand{\alg}{\texorpdfstring{\texttt{$k$NNProxy}}{kNNProxy}\xspace}
\newcommand{\algm}{\texorpdfstring{\texttt{MoP}}{MokNNProxy}\xspace}

\newcommand{\refs}[2]{\hyperref[#1]{\ref*{#1}#2}}

\makeatletter
\DeclareRobustCommand\onedot{\futurelet\@let@token\@onedot}
\def\@onedot{\ifx\@let@token.\else.\null\fi\xspace}

\def\etal{\emph{et al}\onedot}
\makeatother

\DeclareMathOperator*{\argmax}{arg\,max}
\DeclareMathOperator*{\argmin}{arg\,min}

\definecolor{gaintext}{RGB}{140,20,20}
\definecolor{losstext}{RGB}{20,120,20}
\newtcbox{\gaintab}{
on line, 
box align=base, 
colback=WildStrawberry!10,
colframe=orange,
coltext=gaintext,
size=fbox,arc=3pt, 
before upper=\strut, 
top=-2.5pt, 
bottom=-4.5pt, 
left=-2pt, 
right=-2pt, 
boxrule=0pt
}

\newtcbox{\losstab}{
on line,
box align=base,
colback=green!10,
colframe=green!60!black,
coltext=losstext,
size=fbox,arc=3pt,
before upper=\strut,
top=-2.5pt,
bottom=-4.5pt,
left=-2pt,
right=-2pt,
boxrule=0pt
}

\newcommand{\gain}[1]{{\scriptsize\gaintab{\textbf{\textsf{↑#1\%}}}}}
\newcommand{\gainraw}[1]{{\scriptsize\gaintab{\textbf{\textsf{#1}}}}}
\newcommand{\loss}[1]{{\scriptsize\losstab{\textbf{\textsf{↓#1\%}}}}}

\begin{document}
\bstctlcite{IEEEexample:BSTcontrol}

\title{\alg: Efficient Training-Free Proxy Alignment \\ for Black-Box Zero-Shot LLM-Generated Text Detection}

\author{
Kahim~Wong,~Kemou~Li,~Haiwei~Wu,~\IEEEmembership{Member,~IEEE},~and~Jiantao~Zhou,~\IEEEmembership{Senior Member,~IEEE}
\IEEEcompsocitemizethanks{
\IEEEcompsocthanksitem Kahim Wong, Kemou Li, and Jiantao Zhou are with the State Key Laboratory of Internet of Things for Smart City and the Department of Computer and Information Science, Faculty of Science and Technology, University of Macau, Macao 999078, China (e-mail: yc37437@um.edu.mo; kemou.li@connect.umac.mo; jtzhou@umac.mo). \emph{(Corresponding author: Jiantao Zhou.)}

Haiwei Wu is with the School of Information and Software Engineering, University of Electronic Science and Technology of China, Chengdu 611731, China (e-mail: haiweiwu@uestc.edu.cn).
}
}


\maketitle

\begin{abstract}
LLM-generated text (LGT) detection is essential for reliable forensic analysis and for mitigating LLM misuse. Existing LGT detectors can generally be categorized into two broad classes: learning-based approaches and zero-shot methods. Compared with learning-based detectors, zero-shot methods are particularly promising because they eliminate the need to train task-specific classifiers.
However, the reliability of zero-shot methods fundamentally relies on the assumption that an off-the-shelf proxy LLM is well aligned with the often unknown source LLM, a premise that rarely holds in real-world black-box scenarios.
To address this discrepancy, existing proxy alignment methods typically rely on supervised fine-tuning of the proxy or repeated interactions with commercial APIs, thereby increasing deployment costs, exposing detectors to silent API changes, and limiting robustness under domain shift.
Motivated by these limitations, in this work we propose the $k$-nearest neighbor proxy (\alg), a training-free and query-efficient proxy alignment framework that repurposes the \knn language model (\knn-LM) retrieval mechanism as a domain adapter for a fixed proxy LLM. 
Specifically, a lightweight datastore is constructed once from a target-reflective LGT corpus, either via fixed-budget querying or from existing datasets. 
During inference, nearest-neighbor evidence induces a token-level predictive distribution that is interpolated with the proxy output, yielding an aligned prediction without proxy fine-tuning or per-token API outputs.
To improve robustness under domain shift, we extend \alg into a mixture of proxies (\algm)  that routes each input to a domain-specific datastore for domain-consistent retrieval. 
Moreover, we theoretically derive a corpus-size-dependent approximation error bound that motivates token-wise adaptive hyperparameters for more reliable detection without costly hyperparameter sweeps.
Extensive experiments demonstrate strong detection performance of our method, with an average AUROC of 0.99 across eight recent proprietary LLMs and a 6.45\% improvement in AUROC over the prior state-of-the-art (SOTA) alignment method. 
Beyond binary detection, \alg also provides strong utility for closed-set LLM source attribution with an average ACC of 0.84. The source code is available at \url{https://github.com/KahimWong/kNNProxy}.  
\end{abstract}

\begin{IEEEkeywords}
Large language models, LLM-generated text detection, zero-shot detection, \knn-LM
\end{IEEEkeywords}

\section{Introduction}

\IEEEPARstart{L}{arge} language models (LLMs) have achieved remarkable progress, as exemplified by GPT-5.4~\cite{openai_gpt5_4}, Gemini 3.1~\cite{google2026gemini31}, and Claude 4.6~\cite{anthropic2026opus46}, 
which feature advanced multimodal and reasoning capabilities. 
Such progress enables a wide range of applications across domains, including customer service~\cite{jiang2025chatmap}, healthcare diagnostics~\cite{fan2025ai}, and scientific research~\cite{lightman2023let}. 
However, these advances also introduce critical safety concerns: hallucination, prompt injection vulnerability, and privacy leakage are commonly observed in LLMs, and such issues can result in misinformation and adversarial misuse~\cite{ra2025halogen, liu2025compromising, oh2025amplifying, deng2025hardening, he2025defense,li2025llm,li2026aegis,wong2025adcd,chen2026cascade}.
In particular, LLM-generated text (LGT) poses challenges to trust in digital communication, academic integrity, and information authenticity across society~\cite{zhao2024silent}. 
Consequently, LGT detection, which aims to determine whether a given suspicious text is generated by an LLM or written by a human, emerges as an important challenge for mitigating misuse risks and supporting trustworthy forensics~\cite{wu2026editprint,tu2026featdistill}.

\begin{figure}[tb]
\centering
\includegraphics[width=\linewidth]{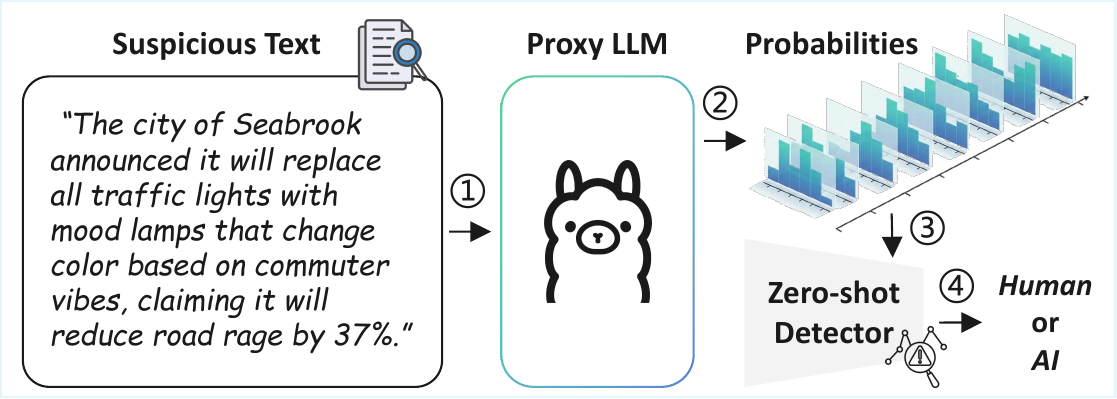}
\vspace*{-5mm}
\caption{\textbf{Pipeline of zero-shot LGT detection.} 
\ding{172}\;Given a suspicious text as input, \ding{173}\;a proxy LLM outputs token-level probabilities. 
\ding{174}\;A zero-shot detector uses them to \ding{175}\;produce a value indicating human/AI generation.}
\label{fig:det_pipeline}
\end{figure}

Existing LGT detection methods identify distinctive characteristics of LGT through either supervised learning~\cite{solaiman2019release} or zero-shot approaches~\cite{bao2024fast, hans2024spotting, xu2025trainingfree}. Learning-based detectors rely on labeled human–LLM data to train binary classifiers~\cite{li2024rml,li2025rml++}, whereas zero-shot detectors primarily exploit statistical properties of LLMs and generally avoid training task-specific classifiers (cf. Sec.~\ref{sec:related-work-lgt} for detailed related works).
As illustrated in Fig.~\ref{fig:det_pipeline}, zero-shot detection employs a proxy LLM to simulate the source LLM (the LLM responsible for generating the given text) and determines whether the text is LGT based on statistical patterns exhibited by the proxy LLM output.
Notably, zero-shot detection is highly sensitive to how accurately the proxy approximates the source LLM; accordingly, strong performance is usually observed only when the proxy output distributions are closely aligned~\cite{mi2024smaller, shi2025phantomhunter} with the source.
When the source LLM is known, as in white-box zero-shot detection, such alignment can be readily achieved. 
In contrast, under the more practical black-box setting where the source LLM is unknown, a single fixed proxy LLM is often insufficient, which makes effective alignment substantially more challenging.

To narrow the alignment gap in black-box zero-shot detection, proxy alignment methods seek to align the proxy LLM with the output distribution of strong proprietary LLMs (i.e., closed-source and commercially owned LLMs).
For instance, as illustrated in Fig.~\ref{fig:proxy_align}, DALD~\cite{zeng2024dald} aligns the proxy via supervised fine-tuning on text generated by a proprietary LLM, while Glimpse~\cite{bao2025glimpse} leverages top-$k$ probabilities returned from commercial APIs to reconstruct an approximate token-level probability distribution that mirrors proprietary model behavior.
Despite their effectiveness, existing prevalent methods encounter two practical obstacles that hinder real-world deployment: \emph{cost} and \emph{robustness under domain shift}.
DALD introduces considerable computational overhead as a result of supervised fine-tuning, while Glimpse is often query-intensive due to repeated interactions with commercial APIs, which expose the method to silent model updates and evolving API behaviors~\cite{glimpse_issue}.
More critically, both approaches exhibit limited robustness on out-of-domain (OOD) text.
DALD typically requires a separately fine-tuned proxy for each domain, whereas Glimpse lacks an explicit mechanism to recalibrate its approximation when the input distribution deviates from the queried data, yielding unreliable alignment on OOD inputs.
Taken together, the limitations above motivate the need for a reusable alignment module that can be built with a fixed budget and deployed offline, while remaining adaptable across domains without re-training or continual API access.

\begin{figure}[tb]
\centering
\includegraphics[width=\linewidth]{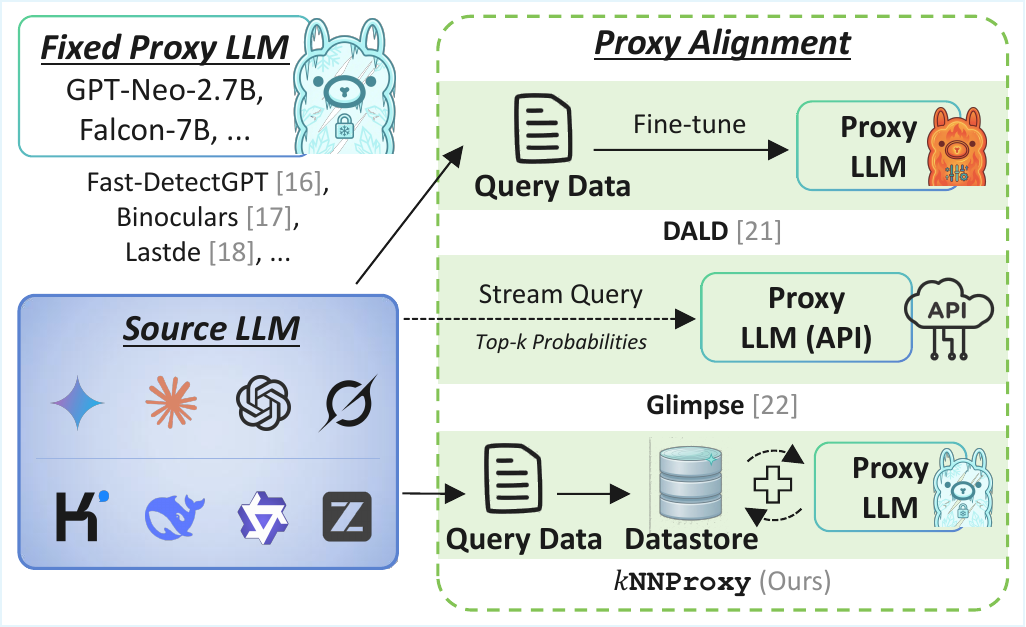}
\vspace*{-5mm}
\caption{\textbf{Illustration of proxy alignment.} 
Fixed-proxy detectors (e.g., Fast-DetectGPT~\cite{bao2024fast}, Binoculars~\cite{hans2024spotting}, Lastde~\cite{xu2025trainingfree}) rely on an off-the-shelf proxy LLM, while proxy alignment narrows the proxy–source gap using queried data. Prior methods align the proxy either by fine-tuning on queried outputs (DALD~\cite{zeng2024dald}) or by streaming API returns (Glimpse~\cite{bao2025glimpse}). Our \alg instead augments a fixed proxy with a datastore built from a fixed query budget, avoiding fine-tuning and repeated API calls while improving detection.}
\label{fig:proxy_align}
\end{figure}

In this paper, we propose \alg, a \textit{training-free} and \textit{query-efficient} proxy alignment approach that repurposes \knn-LM~\cite{kh2020generalization} as a retrieval-based domain adapter for a fixed proxy LLM.
\knn-LM augments a parametric LM with a non-parametric datastore of contexts and next tokens. During inference, nearest-neighbor retrieval induces a token-level predictive distribution that can be interpolated into the proxy output.
Intuitively, \alg works because under similar contexts, the proxy and source LLMs differ mainly in token preference and confidence. The datastore stores source-reflective context--next-token evidence, so retrieval provides a lightweight non-parametric correction signal that nudges proxy probabilities toward source-like behavior exactly where zero-shot detectors are most sensitive.
Rather than using the retrieval mechanism to improve generation, as is typically done in \knn-LM, we use it to calibrate detection-time likelihood patterns. This design explains why \alg can improve robustness under domain shift: when proxy priors become unreliable, retrieved source-reflective neighbors still provide anchored token-level evidence.
Specifically, we first construct a datastore from an LGT corpus that reflects the target model, using either fixed-budget querying or existing datasets. 
Given a suspicious text, we retrieve neighbors at each token position, form a \knn-induced distribution based on their distances and next-tokens, and interpolate this distribution with the proxy output to obtain an aligned prediction.
\alg yields proxy alignment without fine-tuning or per-input API returns at test time, and it supports scalable domain control by swapping datastores without training multiple proxies. Table~\ref{tab:method_comparison} summarizes comparisons between \alg, fixed-proxy baselines, and prior proxy alignment methods.

%

\begin{table}[t]
\centering
\caption{Comparison of black-box zero-shot LGT detection methods.}
\setlength{\tabcolsep}{5.5pt}
\scalebox{0.75}{%
\begin{tabular}{lccccc}
\toprule
\textbf{Method} &
\small \makecell{\textbf{Proxy}\\\textbf{Alignment}} & \small \makecell{\textbf{Training-}\\\textbf{free}} &
\small \makecell{\textbf{Query}\\\textbf{Complexity}} &
\small \makecell{\textbf{Domain}\\\textbf{Control}} &
\small \textbf{AUROC} \\
\midrule
Fast-DetectGPT~\cite{bao2024fast} & $\xmark$ & $\cmark$ & N/A   & $\xmark$      & 0.893 \\
DALD~\cite{zeng2024dald}   & $\cmark$     & $\xmark$     & $\mathrm{O} (1)$ & $\cmark$ & 0.923 \\
Glimpse~\cite{bao2025glimpse}   & $\cmark$     & $\cmark$ & $\mathrm{O}(n)$ & $\xmark$     & 0.930 \\
\rowcolor{gray!15} \alg (Ours)   & $\cmark$   & $\cmark$ & $\mathrm{O}(1)$ & $\cmark$ & \textbf{0.990} \\
\bottomrule
\end{tabular}%
}
\label{tab:method_comparison}
\end{table}

To explicitly improve OOD robustness, we further extend \alg to \algm, an MoE alignment scheme tailored for domain shift.
Since a single datastore can become unreliable when test sample drifts away from its indexing datastore, mismatched neighbors may inject misleading next-token evidence and degrade alignment.
\algm mitigates this failure mode by maintaining multiple domain-specific datastores as retrieval experts and routing each input to the most compatible expert, enforcing domain-consistent retrieval at test time and yielding more stable proxy--source alignment under domain shift.
We further provide a corpus-size-dependent theoretical bound on the approximation error between the source distribution and the aligned distribution. The analysis clarifies when retrieval evidence is trustworthy and directly motivates token-wise adaptive hyperparameters (neighbor size $k$ and interpolation weight $\lambda$) to improve OOD detection reliability without costly tuning. Beyond LGT detection, the same \alg mechanism can also be instantiated as multiple source LLM \knn proxies for closed-set LLM source attribution, highlighting broader forensic utility of our method.

In summary, the contributions of this work are fourfold:
\begin{itemize}[leftmargin=*]
\item \textbf{Training-free proxy alignment.} We propose \alg, a training-free and query-efficient proxy alignment framework for black-box zero-shot LGT detection that repurposes \knn-LM as a \emph{domain adapter} for a fixed proxy LLM, eliminating proxy fine-tuning and avoiding per-input API querying.
\item \textbf{OOD-robust alignment with domain control.} We introduce \algm, an MoE extension that performs domain-consistent retrieval by routing inputs to domain-specific datastores, enhancing robustness under domain shift.
\item \textbf{Theory-guided adaptive design.} We theoretically derive a corpus-size-dependent approximation error bound for retrieval-based alignment and leverage it to develop token-wise adaptive hyperparameters, improving detection while reducing the need for costly hyperparameter sweeps.
\item \textbf{SOTA detection performance.} Extensive experiments show that our method achieves strong detection accuracy, reaching an average AUROC of 0.99 across eight recent proprietary LLMs and improving over the prior SOTA proxy alignment baseline by 6.45\% in AUROC.
\end{itemize}

\subsubsection*{Notations} 
Given a text $\mathbf{x}=[x_1,\ldots,x_T]$ of length $T$, we prepend a begin-of-sequence token $x_0$ and denote the prefix before position $i$ by $\mathbf{x}_{<i}=[x_0,\ldots,x_{i-1}]$.
For any autoregressive language model $M$, we write $\bm{\pi}_{M}^{(i)}=\bm{\pi}_{M}(\cdot \mid \mathbf{x}_{<i}) \in \Delta^{|\mathcal{V}|-1}$ for its next-token distribution at position $i$, where $\mathcal{V}$ is the vocabulary and $\Delta^{|\mathcal{V}|-1}$ is the probability simplex.
The scalar probability assigned to token $v \in \mathcal{V}$ is denoted by $\pi_M^{(i)}(v)$, namely, the $v$-th entry of $\bm{\pi}_M^{(i)}$.

\subsubsection*{Organization} 
The rest of this paper is organized as follows. 
Sec.~\ref{sec:related} reviews related works on learning-based and zero-shot LGT detection. 
Sec.~\ref{sec:method} describes \alg and \algm, including datastore construction, retrieval-interpolation alignment, and the MoE-style router for domain control. 
Sec.~\ref{sec:theory} derives an approximation error bound for the retrieval-induced distribution and motivates token-adaptive hyperparameters. 
Sec.~\ref{sec:experiment} reports experimental results and ablations. 
Finally, Sec.~\ref{sec:conclusion} concludes the paper.

\section{Related Work}\label{sec:related}

\subsection{LGT Detection}
\label{sec:related-work-lgt}

Unlike LLM watermarking, which proactively embeds detectable signals into generated text~\cite{wong2025end,wong2025fontguard}, LGT detection identifies LGT without modifying the LLM output, enabling post hoc use on deployed LLMs. 
Existing passive LGT detection methods are typically grouped into two families: learning-based and zero-shot. 
Learning-based detectors, such as RoBERTa~\cite{solaiman2019release} and XLM-RoBERTa~\cite{conneau2020unsupervised}, identify LGT by training a binary classifier on large-scale datasets containing both human-written text (HWT) and LGT.
However, such methods often exhibit limited OOD generalization and tend to degrade noticeably on unseen topics or texts produced by previously unseen source LLMs~\cite{wu2024detectrl,li2024dat,li2025toward}.

Zero-shot LGT detectors can be further categorized into statistics-, perturbation-, and rewrite-based approaches. 
(1)~Statistics-based methods use a proxy LLM to approximate the source model and compute statistical metrics from the proxy, such as log-likelihood~\cite{ge2019gltr}, entropy~\cite{la2008detecting}, rank~\cite{ge2019gltr}, log-rank~\cite{ge2019gltr}, and LRR~\cite{su2023detectllm}, as surrogates for statistics of the unobserved source model. 
Binoculars~\cite{hans2024spotting} leverages two LLMs and scores text using the ratio of perplexity to cross-perplexity. 
Lastde~\cite{xu2025trainingfree} analyzes log-likelihood sequences by mining multi-scale diversity entropy. 
(2)~Perturbation-based methods, exemplified by DetectGPT~\cite{mitchell2023detectgpt}, randomly perturb the test sample and distinguishes LGT by comparing log-likelihoods of the original and perturbed through probability curvature. 
NPR~\cite{su2023detectllm} follows a similar spirit by using the ratio of log-ranks before and after perturbation. 
Fast-DetectGPT~\cite{bao2024fast} further improves the efficiency of DetectGPT by replacing explicit perturbations with conditional probability computations. 
(3)~Rewrite-based methods adopt a different strategy by comparing an input text with regenerated variants.
Revise-Detect~\cite{zhu2023beat} and RADAR~\cite{hu2023radar} build on the assumption that LLMs revise their own generations less extensively than HWT, and detect LGT by examining discrepancies induced by rewriting. DNA-GPT~\cite{yang2024dna} measures N-gram similarity between truncated original text and regenerated text.

A central limitation of zero-shot detection lies in its strong dependence on the proxy model. 
Detection performance remains strong only when the proxy distribution closely matches the source LLM, and degrades as the mismatch increases~\cite{mi2024smaller,shi2025phantomhunter}. 
To mitigate distributional mismatch, several proxy alignment strategies have been explored.
DALD~\cite{zeng2024dald} fine-tunes the proxy model using text generated from a stronger LLM (e.g., GPT-4) to better align the proxy distribution with the target. 
In addition, since many commercial black-box APIs provide access to next-token probability, Glimpse~\cite{bao2025glimpse} reconstructs the entire probability distribution from the top-$k$ probability returned by proprietary models to support efficient detection. 
Despite these advances, zero-shot detectors still often exhibit high false-positive rates, and their reliability further declines under post-editing or intentional evasion~\cite{nicks2024language,wang2025humanizing}, leaving substantial room for further improvement.

\subsection{\knn-LM}
\label{sec:related-work-knnlm}

\knn-LM~\cite{kh2020generalization} improves the generalizability of pre-trained LLMs by interpolating parametric predictions with non-parametric retrievals from an external datastore, combining implicit model knowledge with explicit memory. 
This retrieval-augmented design enables flexible, training-free adaptation, making it suitable for dynamic and resource-constrained settings. 
Building on \knn-LM, subsequent work focuses on scaling retrieval and broadening applicability. 
He \etal~\cite{he2021efficient} develop efficient search and caching to reduce inference cost on large datastores. 
Huang \etal~\cite{huang2023k} design plug-and-play adapters that enable rapid domain adaptation for black-box LLMs without updating parameters. 
Other studies investigate when and why \knn-LM is effective, emphasizing the role of lexical and semantic matching~\cite{drozdov2022you} and clarifying its advantages over purely parametric LMs~\cite{xu2023nearest}. Beyond standard LLM decoding, \knn-based retrieval has been used to guide speculative decoding~\cite{li2024nearest}. Recent findings show that \knn-LM excels on memory-intensive tasks driven by recurring patterns but is less helpful for multi-step reasoning that requires composing new knowledge~\cite{geng2025great}. Further evidence indicates that \knn-LM's gains concentrate on high-frequency tokens, with limited benefits for rare tokens even under long-tail datastores~\cite{nishida2025long}.
 
In contrast to prior studies that apply \knn-LM to improve generation quality, perplexity, or domain adaptation, our work repurposes \knn-LM for LGT detection, a discriminative and forensic task rather than a generative one.
Here, retrieval supports evidence-driven source analysis instead of improving text generation, thus extending retrieval-augmented modeling to a new application domain.
The next section presents our proposed method in detail.

\begin{figure*}[t]
\centering
\includegraphics[width=\linewidth]{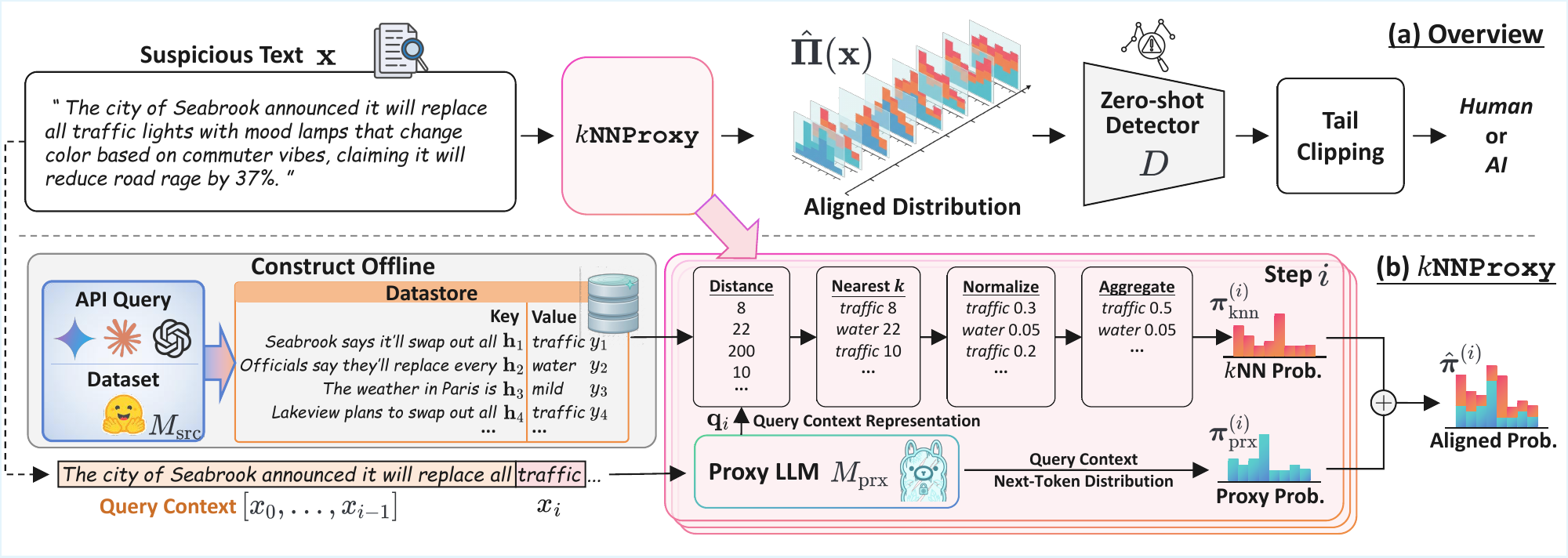}
\vspace*{-5mm}
\caption{
\textbf{Structure of \alg.} (a) The \alg converts a suspicious text into aligned token-level probability, applies a zero-shot detector, and uses tail-clipping aggregation to make a final human-written vs.\ LLM-generated decision. (b) The \alg module aligns a proxy model to a stronger model by retrieving nearest-context examples from an offline datastore and blending the resulting \knn output with the proxy output.
}
\label{fig:model_overview}
\end{figure*}

\section{Methodology}
\label{sec:method}

Fig.~\refs{fig:model_overview}{a} illustrates the overall workflow of \alg.
Given a suspicious text sequence $\mathbf{x}=[x_1,\ldots,x_T]$, \alg produces aligned token-level next-token distributions $\hat{\bm{\Pi}}(\mathbf{x}) = \bigl[\hat{\bm{\pi}}^{(1)},\ldots,\hat{\bm{\pi}}^{(T)}\bigr]$, where each $\hat{\bm{\pi}}^{(i)} \in \Delta^{|\mathcal{V}|-1}$ denotes the aligned prediction at position $i$.
The sequence is fed into an off-the-shelf zero-shot detector, which evaluates the observed token sequence through the aligned token log-likelihoods $\ell_i = \log \hat{\pi}^{(i)}(x_i)$ for $i=1,\dots,T$.
A lower-bound clipping rule over $\{\ell_i\}_{i=1}^{T}$ is used to reduce outlier effects and derive the final decision statistic for distinguishing HWT from LGT.

Fig.~\refs{fig:model_overview}{b} details how \alg aligns a fixed proxy LLM $M_{\mathrm{prx}}$ with a stronger proprietary source LLM $M_{\mathrm{src}}$.
Specifically, a source-reflective corpus is first collected via querying $M_{\mathrm{src}}$ or reusing existing LGT data, followed by constructing a datastore of context--next-token pairs. 
At inference, for each position $i$, we compute the proxy distribution $\bm{\pi}_{\mathrm{prx}}^{(i)}$, retrieve $k$ nearest neighbors from the datastore, and form a retrieval-induced distribution $\bm{\pi}_{\mathrm{knn}}^{(i)}$.
$\bm{\pi}_{\mathrm{prx}}^{(i)}$ and $\bm{\pi}_{\mathrm{knn}}^{(i)}$ are interpolated to produce the aligned distribution $\hat{\bm{\pi}}^{(i)}$.

\subsection{\alg: \knn-LM for Proxy LLM Alignment}
\label{sec:knnlm4align}

The goal of \alg is to approximate the conditional next-token distribution of the unknown source LLM $M_{\mathrm{src}}$ by a retrieval-augmented proxy, without accessing the source model internals.
Formally, at token position $i$, we aim to approximate the source distribution $\bm{\pi}_{\mathrm{src}}^{(i)}=\bm{\pi}_{M_{\mathrm{src}}}(\cdot \mid \mathbf{x}_{<i})$.
To this end, \alg builds an offline datastore from a large corpus generated by, or otherwise reflective of, $M_{\mathrm{src}}$, as illustrated in Fig.~\refs{fig:model_overview}{b}.
This corpus can be obtained either by prompting a black-box LLM through APIs or by leveraging existing LGT datasets.
Importantly, \textbf{\textit{our method does not require the internal parameters or token-level output probabilities of the source LLM}}, which makes it more robust to API changes compared to Glimpse.\footnote{The default Glimpse relies on access to top-probability outputs from GPT-4 APIs, which are no longer available in the current OpenAI API~\cite{glimpse_issue}.}

\subsubsection{Offline Datastore Construction}
Let $\mathcal{C}_{\mathrm{src}}$ denote the collected source-reflective corpus.
Using a sliding window of length $W$, we construct a datastore $\mathcal{I} = \{(\mathbf{h}_n, y_n)\}_{n=1}^{N}$, where $\mathbf{h}_n \in \mathbb{R}^{d}$ is a context representation and $y_n \in \mathcal{V}$ is its corresponding successor token.
Concretely, for each local context window $\mathbf{c}$ extracted from the corpus, we encode it with the fixed proxy model $M_{\mathrm{prx}}$ and denote the resulting feature vector by $\mathbf{h} = \psi(\mathbf{c}; M_{\mathrm{prx}}) \in \mathbb{R}^{d}$, where $\psi(\cdot; M_{\mathrm{prx}})$ is the chosen hidden-state feature extractor of the proxy model.
The paired value is the next token immediately following $\mathbf{c}$.
After scanning the whole corpus, all keys $\{\mathbf{h}_n\}_{n=1}^{N}$ are indexed using FAISS~\cite{johnson2019billion} to enable efficient nearest-neighbor retrieval.

\subsubsection{Retrieval-induced Distribution}
Given a suspicious text $\mathbf{x}$, \alg\ is applied token by token.
At position $i$, we first form the query context $\mathbf{c}_i = [x_0,\ldots,x_{i-1}]$, and compute its query representation $\mathbf{q}_i=\psi(\mathbf{c}_i; M_{\mathrm{prx}}) \in \mathbb{R}^{d}$.
We then retrieve the index set of its $k$ nearest neighbors from the datastore, $\mathcal{N}_k(\mathbf{q}_i) \subseteq \{1,\ldots,N\}$.
For each retrieved entry $j \in \mathcal{N}_k(\mathbf{q}_i)$, let $\mathbf{h}_j$ be the stored key and $y_j$ the stored next token.
We assign a distance-based weight
\begin{equation}\label{eq:weight}
\alpha_{ij}
=
\frac{
\exp\bigl(-\mathrm{dist}(\mathbf{q}_i,\mathbf{h}_j)/\tau\bigr)
}{
\sum_{l \in \mathcal{N}_k(\mathbf{q}_i)}
\exp\bigl(-\mathrm{dist}(\mathbf{q}_i,\mathbf{h}_{l})/\tau\bigr)
},
\ \ \ \  j \in \mathcal{N}_k(\mathbf{q}_i),
\end{equation}
where $\mathrm{dist}(\cdot,\cdot)$ denotes the chosen distance metric and $\tau>0$ is a temperature parameter.
The normalized weights satisfy $\sum_{j \in \mathcal{N}_k(\mathbf{q}_i)} \alpha_{ij}=1$.
Based on the retrieved neighbors, we define the \knn-induced next-token distribution as
\begin{equation}
\pi_{\mathrm{knn}}^{(i)}(v)
=
\sum\nolimits_{j \in \mathcal{N}_k(\mathbf{q}_i)}
\alpha_{ij}\,\mathbbm{1}\{y_j = v\},
\quad v \in \mathcal{V},
\end{equation}
and denote the corresponding probability vector by
\[
\bm{\pi}_{\mathrm{knn}}^{(i)}
=
\bigl[\pi_{\mathrm{knn}}^{(i)}(v)\bigr]_{v \in \mathcal{V}}.
\]

\subsubsection{Aligned Distribution}
Meanwhile, the fixed proxy model $M_{\mathrm{prx}}$ produces its own next-token distribution $\bm{\pi}_{\mathrm{prx}}^{(i)} = \bm{\pi}_{M_{\mathrm{prx}}}(\cdot \mid \mathbf{x}_{<i})$.
We then interpolate the parametric proxy prediction and the non-parametric retrieval prediction:
\begin{equation}
\hat{\bm{\pi}}^{(i)}
=
\lambda_i \bm{\pi}_{\mathrm{prx}}^{(i)}
+
(1-\lambda_i) \bm{\pi}_{\mathrm{knn}}^{(i)},
\qquad \lambda_i \in [0,1].
\label{eq:aligned_dist}
\end{equation}
Here $\lambda_i$ is the interpolation weight.
In the basic \alg formulation, one may simply set $\lambda_i=\lambda$ as a global constant; later in Sec.~\ref{sec:theory} we further generalize it to a token-adaptive form.
Repeating Eq.~\eqref{eq:aligned_dist} for $i=1,\ldots,T$ yields the aligned distribution sequence $\hat{\bm{\Pi}}(\mathbf{x})$.

\subsubsection{Zero-shot Detection with Aligned Distributions}
Because many zero-shot detectors are fundamentally distribution-based, $\hat{\bm{\Pi}}(\mathbf{x})$ serves as a source-aware proxy for the unknown source model and can therefore strengthen black-box LGT detection without proxy fine-tuning or repeated API calls at test time.
Below we reformulate Fast-DetectGPT and Binoculars using the aligned distributions.
For convenience, define the aligned token log-likelihood score of any text $\mathbf{x}=[x_1,\ldots,x_{|\mathbf{x}|}]$ as
\begin{equation}
s_{\mathrm{alg}}(\mathbf{x})
=
\sum\nolimits_{i=1}^{|\mathbf{x}|}
\log \hat{\pi}^{(i)}(x_i).
\end{equation}

\paragraph{Fast-DetectGPT with \alg}
Let $M_{\mathrm{ref}}$ be the reference model used by Fast-DetectGPT.
Denote by $P_{M_{\mathrm{ref}}}(\tilde{\mathbf{x}}\mid \mathbf{x})$ the conditional sampling distribution used by the reference model to generate comparison samples $\tilde{\mathbf{x}}$ from the same context as $\mathbf{x}$.
We define
\begin{equation}
D_{\mathrm{Fast}}(\mathbf{x}; M_{\mathrm{ref}})
=
\frac{
s_{\mathrm{alg}}(\mathbf{x}) - \mu_{\mathrm{ref}}(\mathbf{x})
}{
\sigma_{\mathrm{ref}}(\mathbf{x})
},
\end{equation}
where
\begin{equation}
\begin{aligned}
\mu_{\mathrm{ref}}(\mathbf{x}) 
&= \mathbb{E}_{\tilde{\mathbf{x}} \sim P_{M_{\mathrm{ref}}}(\cdot \mid \mathbf{x})}[s_{\mathrm{alg}}(\tilde{\mathbf{x}})], \\
\sigma_{\mathrm{ref}}^{2}(\mathbf{x}) 
&= \mathbb{E}_{\tilde{\mathbf{x}} \sim P_{M_{\mathrm{ref}}}(\cdot \mid \mathbf{x})}\left[(s_{\mathrm{alg}}(\tilde{\mathbf{x}})-\mu_{\mathrm{ref}}(\mathbf{x}))^2\right].
\end{aligned}
\end{equation}

\paragraph{Binoculars with \alg}
Binoculars compares the self-likelihood under the scoring model against the cross-entropy induced by a reference model.
With the aligned scoring distribution, define the per-token negative log-likelihood
\begin{equation}
\mathsf{NLL}_{\mathrm{alg}}(\mathbf{x})
=
-\frac{1}{T}
\sum\nolimits_{i=1}^{T}
\log \hat{\pi}^{(i)}(x_i),
\end{equation}
and let $\bm{\pi}_{\mathrm{ref}}^{(i)} = \bm{\pi}_{M_{\mathrm{ref}}}(\cdot \mid \mathbf{x}_{<i})$ be the next-token distribution of the reference model at position $i$.
We define the aligned cross-entropy with respect to the reference model as
\begin{equation}
\mathsf{H}_{\mathrm{alg},\mathrm{ref}}(\mathbf{x})
=
-\frac{1}{T}
\sum\nolimits_{i=1}^{T}
\sum\nolimits_{v \in \mathcal{V}}
\hat{\pi}^{(i)}(v)\log \pi_{\mathrm{ref}}^{(i)}(v).
\end{equation}
Accordingly, the Binoculars score becomes
\begin{equation}
D_{\mathrm{Bino}}(\mathbf{x}; M_{\mathrm{ref}})
=
\frac{
\exp\bigl(\mathsf{NLL}_{\mathrm{alg}}(\mathbf{x})\bigr)
}{
\exp\bigl(\mathsf{H}_{\mathrm{alg},\mathrm{ref}}(\mathbf{x})\bigr)
}.
\end{equation}

The resulting detector score $D$ is finally compared against a predefined threshold to make the HWT/LGT decision.

\subsection{\algm: Multi-Domain Proxy Alignment} 
\label{sec:moknnproxy}

\begin{figure}[t]
\centering
\includegraphics[width=\linewidth]{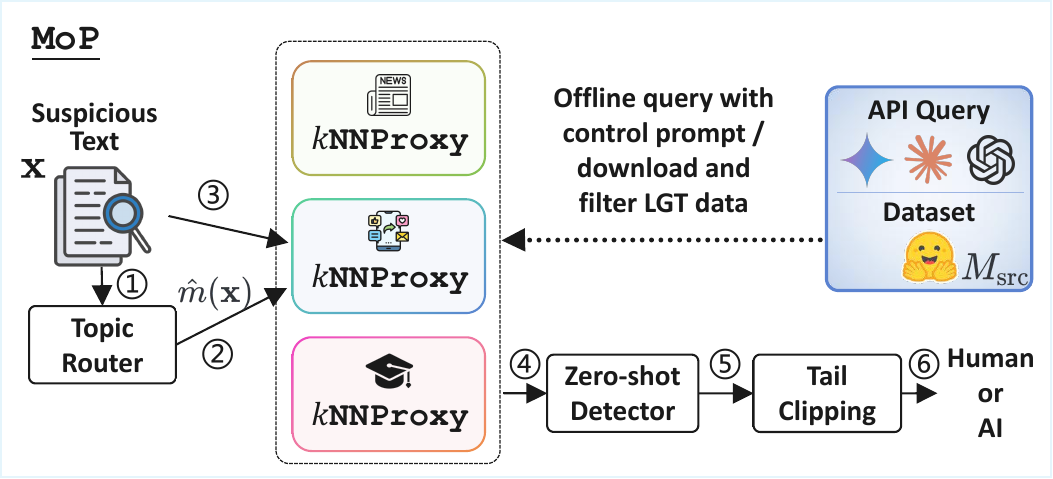}
\vspace*{-5mm}
\caption{
\textbf{Schematic diagram of \algm.} Mixture of proxies involves multiple domain-specific \alg experts, and a lightweight router selects the most relevant expert for each input text.
}
\label{fig:moe}
\end{figure}

While a single \alg effectively captures domain-specific patterns for LGT detection, it often underperforms on OOD samples due to limited generalization. 
To improve robustness under domain shift, we extend \alg to an MoE variant, \algm, where multiple domain-specific datastores are constructed offline and a lightweight router selects the most compatible expert.
As shown in Fig.~\ref{fig:moe}, each expert corresponds to a domain, defined by topic (e.g., academic writing, social media, or news) or source-model family.
Since datastore construction in \alg is training-free and modular, such multi-domain extension is straightforward.

Formally, let $\{\mathcal{I}_m\}_{m=1}^{M}$ be $M$ domain-specific datastores, where $\mathcal{I}_m$ is built from texts belonging to domain $m$ and instantiates the $m$-th \alg\ expert.
To select the most compatible expert for a test sample, we build an auxiliary routing datastore.
Specifically, from all domain corpora we collect a sentence--domain set $\{(\mathbf{s}_n, z_n)\}_{n=1}^{N_{\mathrm{r}}}$, where $\mathbf{s}_n$ is a sentence and $z_n \in \{1,\ldots,M\}$ is its domain label.
Using an off-the-shelf sentence embedding model $\phi$, we compute $\mathbf{e}_n = \phi(\mathbf{s}_n) \in \mathbb{R}^{d_{\mathrm{r}}}$, and form the routing datastore $\mathcal{I}_{\mathrm{r}} = \{(\mathbf{e}_n, z_n)\}_{n=1}^{N_{\mathrm{r}}}$.
Given a suspicious text $\mathbf{x}$, we first compute its sentence embedding $\mathbf{e}(\mathbf{x}) = \phi(\mathbf{x})$, and retrieve its $k_{\mathrm{r}}$ nearest neighbors in the routing datastore:
\begin{equation}
\mathcal{N}^{\mathrm{r}}_{k_{\mathrm{r}}}\bigl(\mathbf{e}(\mathbf{x})\bigr)
\subseteq
\{1,\ldots,N_{\mathrm{r}}\}.
\end{equation}
We then estimate an empirical routing score for each domain $m$ by majority voting:
\begin{equation}
\omega_m(\mathbf{x})
=
\frac{1}{k_{\mathrm{r}}}
\sum\nolimits_{n \in \mathcal{N}^{\mathrm{r}}_{k_{\mathrm{r}}}(\mathbf{e}(\mathbf{x}))}
\mathbbm{1}\{z_n = m\},
\end{equation}
where $m=1,\dots,M$. The routed expert index is chosen as
\begin{equation}
\hat{m}(\mathbf{x})
=
\argmax_{m \in \{1,\ldots,M\}}
\omega_m(\mathbf{x}).
\end{equation}
Finally, we apply the expert associated with $\mathcal{I}_{\hat{m}(\mathbf{x})}$ to produce aligned token distributions and perform downstream detection.

Although one could organize experts by source LLM, we prefer domain/topic partitioning in practice.
The reason is that topic structure is often more stable than the rapidly evolving set of commercial LLM versions, and topic-specific retrieval tends to transfer better across unseen or updated source models.
As demonstrated in Sec.~\ref{sec:experiment}, this design yields accurate routing and improved OOD robustness for LGT detection.

\subsection{Robust Lower-Tail Clipping for Detection Stability}
\label{sec:tailclip}

\begin{figure}[tb] 
\centering 
\includegraphics[width=\linewidth]{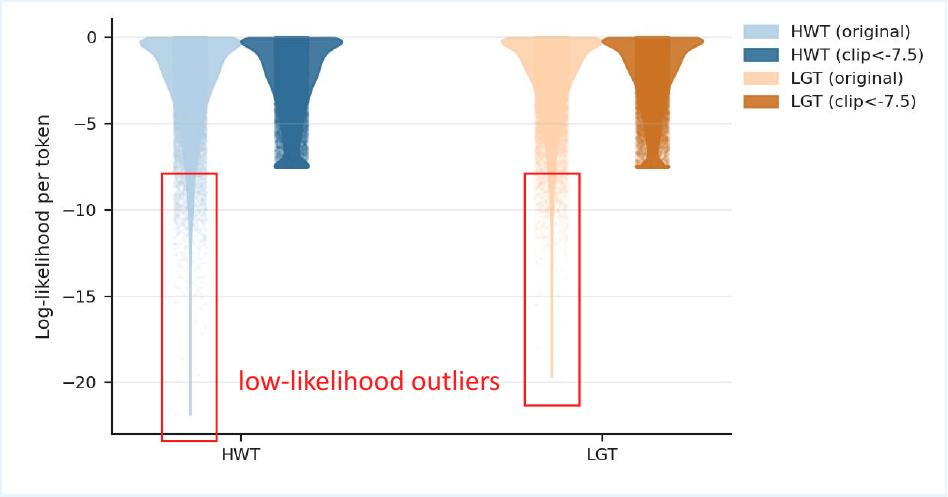} 
\vspace*{-5mm}
\caption{\textbf{Per-token log-likelihood distributions} for 100 human-written and GPT-4-generated texts using a GPT-Neo-2.7B proxy, w/wo tail clipping.} \label{fig:tail_clip} 
\end{figure}

Beyond improving proxy--source alignment, we further stabilize zero-shot detection by replacing the standard arithmetic mean over token scores with a robust lower-bound clipping strategy.
Many existing detectors, e.g., likelihood~\cite{ge2019gltr}, Fast-DetectGPT~\cite{bao2024fast}, and Binoculars~\cite{hans2024spotting}, aggregate token-level evidence by averaging.
However, token log-likelihood sequences often contain extreme outliers for both HWT and LGT.
As shown in Fig.~\ref{fig:tail_clip}, low-likelihood outliers frequently arise from tokens poorly modeled by the proxy, and can dominate the aggregated score, thereby degrading separability.

To mitigate the effect of outliers, we clip each aligned token log-likelihood from below before averaging.
Fig.~\ref{fig:tail_clip} indicates that the dominant outliers are concentrated in the low-likelihood tail; clipping this tail effectively prevents a few extreme tokens from overwhelming the sequence-level score.
Formally, for $\ell_i = \log \hat{\pi}^{(i)}(x_i)$, we define the clipped score
\begin{equation}
\tilde{\ell}_i = \max(\ell_i, \gamma),
\end{equation}
where $\gamma$ is a lower bound.
The final sequence statistic is then computed as
\begin{equation}
\bar{\ell}_{\gamma}
=
\frac{1}{T}
\sum\nolimits_{i=1}^{T}
\tilde{\ell}_i.
\label{eq:clipped_mean}
\end{equation}
In our experiments, we use a fixed default lower bound $\gamma$.
This operation preserves most token-level evidence while suppressing only pathological low-likelihood spikes, and can be directly applied to different zero-shot scoring rules with negligible overhead.

\section{Theoretical Analysis of Distribution Discrepancy and Adaptive Parameters in \alg}
\label{sec:theory}

We now analyze how accurately the retrieval-induced distribution in \alg approximates the unknown source distribution.
In Sec.~\ref{sec:errbound}, we first derive a high-probability bound for the discrepancy between the source next-token distribution and the \knn-induced distribution at a fixed token position.
We then use the bound to extract bias and variance surrogates that motivate token-adaptive hyperparameters in Sec.~\ref{sec:adapt_param}.

\subsection{Error Bound of \alg}
\label{sec:errbound}

\textbf{Setup.}
Fix a query representation $\mathbf{q}\in\mathbb{R}^{d}$.
Let $\mathcal{I}=\{(\mathbf{h}_n,y_n)\}_{n=1}^{N}$ be the datastore, where $\mathbf{h}_n\in\mathbb{R}^{d}$ is a context representation and $y_n\in\mathcal{V}$ is its successor token.
For a neighborhood size $k$, let $\mathcal{N}_k(\mathbf{q})\subseteq\{1,\ldots,N\}$ denote the indices of the $k$ nearest neighbors of $\mathbf{q}$ under $L_2$ distance.
Following Eq.~\eqref{eq:weight}, the normalized retrieval weights are
\begin{equation}
\alpha_j(\mathbf{q})
=
\frac{
\exp\left(-\|\mathbf{q}-\mathbf{h}_j\|_2/\tau\right)
}{
\sum\nolimits_{l\in\mathcal{N}_k(\mathbf{q})}
\exp\left(-\|\mathbf{q}-\mathbf{h}_l\|_2/\tau\right)
}, \ \  j\in\mathcal{N}_k(\mathbf{q}),
\label{eq:theory_alpha}
\end{equation}
where $\tau>0$ is the temperature parameter.
The corresponding \knn-induced next-token distribution is
\begin{equation}
\pi_{\mathrm{knn}}(v\mid \mathbf{q})
=
\sum\nolimits_{j\in\mathcal{N}_k(\mathbf{q})}
\alpha_j(\mathbf{q})\,\mathbbm{1}\{y_j=v\},
\ \  v\in\mathcal{V}.
\label{eq:theory_pknn}
\end{equation}

To quantify the locality and weight concentration, define
\begin{equation}
k_{\mathrm{eff}}(\mathbf{q})
 \coloneq 
\frac{1}{\sum\nolimits_{j\in\mathcal{N}_k(\mathbf{q})}\alpha_j(\mathbf{q})^2}
\in [1,k],
\label{eq:theory_keff}
\end{equation}
and the effective neighborhood radius
\begin{equation}
r_{\mathrm{eff}}(\mathbf{q})
 \coloneq 
\sum\nolimits_{j\in\mathcal{N}_k(\mathbf{q})}
\alpha_j(\mathbf{q})\,\|\mathbf{q}-\mathbf{h}_j\|_2.
\label{eq:theory_reff}
\end{equation}
We also define the conditional mean estimator
\begin{equation}
\begin{aligned}
\bar{\pi}_{\mathrm{knn}}(v\mid \mathbf{q})
 \coloneq  {}&
\mathbb{E}\left[
\pi_{\mathrm{knn}}(v\mid \mathbf{q})
\,\middle|\,
\{\mathbf{h}_j\}_{j\in\mathcal{N}_k(\mathbf{q})}
\right]\\
= {}&
\sum\nolimits_{j\in\mathcal{N}_k(\mathbf{q})}
\alpha_j(\mathbf{q})\,\pi_{\mathrm{src}}(v\mid \mathbf{h}_j),
\end{aligned}
\label{eq:theory_condmean}
\end{equation}
where $\bm{\pi}_{\mathrm{src}}(\cdot\mid \mathbf{h})$ denotes the source next-token distribution at representation $\mathbf{h}$.
For brevity, below we write $V\coloneq|\mathcal{V}|$ and
\[
\mathcal{N} \coloneq \mathcal{N}_k(\mathbf{q}), \ 
\alpha_j \coloneq \alpha_j(\mathbf{q}), \ 
k_{\mathrm{eff}} \coloneq k_{\mathrm{eff}}(\mathbf{q}), \ 
r_{\mathrm{eff}} \coloneq r_{\mathrm{eff}}(\mathbf{q}).
\]

\begin{assumption}[Local Smoothness and Conditional Sampling]
\label{ass:theory_main}
The source next-token distribution satisfies:
\begin{enumerate}[leftmargin=*]
\item \textbf{$L_1$-Lipschitz Continuity:}
\[
\bigl\|
\bm{\pi}_{\mathrm{src}}(\cdot\mid \mathbf{h}')
-
\bm{\pi}_{\mathrm{src}}(\cdot\mid \mathbf{h})
\bigr\|_1
\le
L\,\|\mathbf{h}'-\mathbf{h}\|_2,
\quad
\forall\,\mathbf{h},\mathbf{h}'\in\mathbb{R}^{d}.
\]
\item \textbf{Conditional Independence of Retrieved Tokens:}
conditioned on the retrieved keys $\{\mathbf{h}_j\}_{j\in\mathcal N}$, the tokens $\{y_j\}_{j\in\mathcal N}$ are independent and satisfy
\[
y_j \sim \bm{\pi}_{\mathrm{src}}(\cdot\mid \mathbf{h}_j),
\quad j\in\mathcal N.
\]
\end{enumerate}
\end{assumption}

Assum.~\ref{ass:theory_main} is standard and mild in our setting.
Assum.~\refs{ass:theory_main}{.1} states that nearby context representations should induce similar source distributions, which is exactly the smoothness needed for local retrieval to be meaningful.
Assum.~\refs{ass:theory_main}{.2} matches the datastore construction process: once the retrieved contexts are fixed, their next tokens are treated as draws from the corresponding source conditional distributions.

\begin{theorem}[$L_1$ Bound for Weighted \knn Retrieval]
\label{thm:tv_bound}
Under Assum.~\ref{ass:theory_main}, for any $\delta\in(0,1)$, with probability at least $1-\delta$,
\begin{equation}
\bigl\|
\bm{\pi}_{\mathrm{src}}(\cdot\mid \mathbf{q})
-
\bm{\pi}_{\mathrm{knn}}(\cdot\mid \mathbf{q})
\bigr\|_1
\le
Lr_{\mathrm{eff}}
+
V\sqrt{\frac{\log(2V/\delta)}{2\,k_{\mathrm{eff}}}}.
\label{eq:error_bound}
\end{equation}
\end{theorem}

\begin{proof}

Let $\bar{\bm{\pi}}_{\mathrm{knn}}(\cdot\mid\mathbf{q})$ denote the vector form of $\bar{\pi}_{\mathrm{knn}}(v\mid\mathbf{q})$.
By the triangle inequality,
\begin{equation}
\begin{aligned}
\bigl\|
\bm{\pi}_{\mathrm{src}}(\cdot\mid \mathbf{q})
-
\bm{\pi}_{\mathrm{knn}}(\cdot\mid \mathbf{q})
\bigr\|_1
\le
 \underbrace{
\bigl\|
\bm{\pi}_{\mathrm{src}}(\cdot\mid \mathbf{q})
-
\bar{\bm{\pi}}_{\mathrm{knn}}(\cdot\mid \mathbf{q})
\bigr\|_1
}_{\text{bias}}\\
+
\underbrace{
\bigl\|
\bar{\bm{\pi}}_{\mathrm{knn}}(\cdot\mid \mathbf{q})
-
\bm{\pi}_{\mathrm{knn}}(\cdot\mid \mathbf{q})
\bigr\|_1
}_{\text{variance}}.
\end{aligned}
\label{eq:bv_split}
\end{equation}

For the bias term, Eq.~\eqref{eq:theory_condmean} and Assum.~\refs{ass:theory_main}{.1} give
\begin{equation}
\begin{aligned}
&\quad \ \bigl\|
\bm{\pi}_{\mathrm{src}}(\cdot\mid \mathbf{q})
-
\bar{\bm{\pi}}_{\mathrm{knn}}(\cdot\mid \mathbf{q})
\bigr\|_1\\
&=
\left\|
\sum\nolimits_{j\in\mathcal N}
\alpha_j
\bigl(
\bm{\pi}_{\mathrm{src}}(\cdot\mid \mathbf{q})
-
\bm{\pi}_{\mathrm{src}}(\cdot\mid \mathbf{h}_j)
\bigr)
\right\|_1 \\
&\le
\sum\nolimits_{j\in\mathcal N}
\alpha_j
\bigl\|
\bm{\pi}_{\mathrm{src}}(\cdot\mid \mathbf{q})
-
\bm{\pi}_{\mathrm{src}}(\cdot\mid \mathbf{h}_j)
\bigr\|_1 \\
&\le
L\sum\nolimits_{j\in\mathcal N}\alpha_j\|\mathbf{q}-\mathbf{h}_j\|_2
=
L\,r_{\mathrm{eff}}.
\end{aligned}
\label{eq:bias_bound}
\end{equation}

For the variance term, fix $v\in\mathcal V$ and define
\[
\xi_j^{(v)}
 \coloneq 
\mathbbm{1}\{y_j=v\}
-
\pi_{\mathrm{src}}(v\mid \mathbf{h}_j),
\quad j\in\mathcal{N}.
\]
Conditioned on $\{\mathbf{h}_j\}_{j\in\mathcal N}$, variables $\{\xi_j^{(v)}\}_{j\in\mathcal N}$ are independent, mean-zero, and each lies in an interval of length $1$.
Since
\[
\pi_{\mathrm{knn}}(v\mid \mathbf{q})-\bar{\pi}_{\mathrm{knn}}(v\mid \mathbf{q})
=
\sum\nolimits_{j\in\mathcal N}\alpha_j\,\xi_j^{(v)},
\]
Hoeffding's inequality for weighted sums gives
\begin{equation}
\begin{aligned}
& \quad\ \Pr\big(
|
\pi_{\mathrm{knn}}(v\mid \mathbf{q})
-
\bar{\pi}_{\mathrm{knn}}(v\mid \mathbf{q})
|
\ge t
\,\big|\,
\{\mathbf{h}_j\}_{j\in\mathcal N}
\big)\\
&\le
2\exp\left(
-\frac{2t^2}{\sum\nolimits_{j\in\mathcal{N}}\alpha_j^2}
\right)
=
2\exp\bigl(-2k_{\mathrm{eff}}t^2\bigr).
\label{eq:hoeffding_weighted}
\end{aligned}
\end{equation}
If $
\bigl\|
\bm{\pi}_{\mathrm{knn}}(\cdot\mid \mathbf{q})
-
\bar{\bm{\pi}}_{\mathrm{knn}}(\cdot\mid \mathbf{q})
\bigr\|_1
\ge \varepsilon$,
then at least one coordinate must satisfy
\[
\left|
\pi_{\mathrm{knn}}(v\mid \mathbf{q})
-
\bar{\pi}_{\mathrm{knn}}(v\mid \mathbf{q})
\right|
\ge \frac{\varepsilon}{V}.
\]
Applying a union bound over $v\in\mathcal V$ yields
\begin{equation}
\begin{aligned}
&\Pr\left(
\bigl\|
\bm{\pi}_{\mathrm{knn}}(\cdot\mid \mathbf{q})
-
\bar{\bm{\pi}}_{\mathrm{knn}}(\cdot\mid \mathbf{q})
\bigr\|_1
\ge \varepsilon
\,\middle|\,
\{\mathbf{h}_j\}_{j\in\mathcal N}
\right)\\
\le {}&
\sum\nolimits_{v\in\mathcal{V}}
\Pr\left(
\left|
\pi_{\mathrm{knn}}(v\mid \mathbf{q})
-
\bar{\pi}_{\mathrm{knn}}(v\mid \mathbf{q})
\right|
\ge \frac{\varepsilon}{V}
\,\middle|\,
\{\mathbf{h}_j\}_{j\in\mathcal N}
\right) \\
\le{} &
2V\exp\left(
-\frac{2k_{\mathrm{eff}}\varepsilon^2}{V^2}
\right).
\label{eq:union}
\end{aligned}
\end{equation}
Setting the RHS to $\delta$ gives, with probability at least $1-\delta$,
\begin{equation}
\bigl\|
\bm{\pi}_{\mathrm{knn}}(\cdot\mid \mathbf{q})
-
\bar{\bm{\pi}}_{\mathrm{knn}}(\cdot\mid \mathbf{q})
\bigr\|_1
\le
V\sqrt{\frac{\log(2V/\delta)}{2\,k_{\mathrm{eff}}}}.
\label{eq:variance_bound}
\end{equation}
Combining Eqs.~\eqref{eq:bv_split}, \eqref{eq:bias_bound}, and \eqref{eq:variance_bound} proves Eq.~\eqref{eq:error_bound}.
\end{proof}

\begin{remark}
Thm.~\ref{thm:tv_bound} separates the retrieval error into a local bias term $Lr_{\mathrm{eff}}$ and a variance term of order $k_{\mathrm{eff}}^{-1/2}$.
It therefore suggests that retrieval is most accurate when the selected neighbors are both close to the query and not overly concentrated on a few samples.
These two quantities will be the key proxies for designing adaptive hyperparameters in Sec.~\ref{sec:adapt_param}.
\end{remark}

\subsection{Token-adaptive \alg Hyperparameters}
\label{sec:adapt_param}

Thm.~\ref{thm:tv_bound} shows that the retrieval error at query $\mathbf q$ is governed by two local quantities: the effective radius $r_{\mathrm{eff}}(\mathbf q)$ and the effective neighborhood size $k_{\mathrm{eff}}(\mathbf q)$.
This suggests that a single global choice of hyperparameters may be suboptimal across different token positions.
Intuitively, in dense regions, using more neighbors and smoother weights can reduce variance with only a limited increase in bias.
In sparse regions, by contrast, large neighborhoods or overly smooth weights may place nontrivial mass on distant and less relevant keys, thereby increasing bias.
Motivated by this trade-off, we make the neighborhood size $k$, the retrieval temperature $\tau$, and the interpolation weight $\lambda$ adaptive to each query representation.

\textbf{Adaptive $k$ and $\tau$.}
Guided by Eq.~\eqref{eq:error_bound}, we define the following surrogate for the local retrieval error:
\begin{equation}
\mathcal{U}(k,\tau;\mathbf q)
\coloneq
c\,r_{\mathrm{eff}}(k,\tau;\mathbf q)
+
\frac{1}{\sqrt{k_{\mathrm{eff}}(k,\tau;\mathbf q)}},
\label{eq:ret_surrogate}
\end{equation}
where $c>0$ is a global calibration coefficient.
This objective mirrors the bias--variance structure in Thm.~\ref{thm:tv_bound}: the first term penalizes nonlocal retrieval, whereas the second penalizes overly concentrated weights.
Since the constants in Eq.~\eqref{eq:error_bound}, such as $L$, $V$, and $\delta$, are either unknown or fixed across queries, we absorb their relative scaling into $c$.

We then select the token-wise retrieval hyperparameters by
\begin{equation}
(k^{*}(\mathbf q),\tau^{*}(\mathbf q))
\in
\argmin_{k\in\mathcal K,\tau\in\mathcal T}
\mathcal{U}(k,\tau;\mathbf q),
\label{eq:ktau_star}
\end{equation}
where $\mathcal K$ and $\mathcal T$ are predefined candidate sets.
In practice, for each query we retrieve the top-$k_{\max}$ neighbors only once, and then evaluate Eq.~\eqref{eq:ktau_star} over all candidate pairs $(k,\tau)$ using truncated prefixes of this ranked list.
Consequently, stable local regions tend to favor larger $k$ and smoother weighting, while sparse or unreliable regions prefer smaller neighborhoods and sharper weighting.

\textbf{Adaptive $\lambda$.}
After determining $k^{*}(\mathbf q)$ and $\tau^{*}(\mathbf q)$, we use the resulting local retrieval quality to adapt the interpolation weight in Eq.~\eqref{eq:aligned_dist}.
Recall that $\lambda$ weights the proxy distribution $\bm{\pi}_{\mathrm{prx}}^{(i)}$, while $1-\lambda$ weights the retrieval distribution $\bm{\pi}_{\mathrm{knn}}^{(i)}$.
Therefore, a more reliable retrieval estimate should lead to a \emph{smaller} $\lambda$.
Using the optimized retrieval parameters, define
\begin{equation}
\mathcal{U}^{*}(\mathbf q)
\coloneq
\mathcal{U}\bigl(k^{*}(\mathbf q),\tau^{*}(\mathbf q);\mathbf q\bigr).
\label{eq:u_star}
\end{equation}
We then convert this quantity into an adaptive interpolation weight through a sigmoid mapping:
\begin{equation}
\lambda^{*}(\mathbf q)
=
\sigma\bigl(
\mathcal{U}^{*}(\mathbf q)
-
\operatorname{median}\{\mathcal{U}^{*}(\mathbf q'):\mathbf q'\in\mathcal{Q}(\mathbf x)\}
\bigr),
\label{eq:lambda_star}
\end{equation}
where $\sigma(\cdot)$ is the sigmoid function, $\operatorname{median}$ denotes the median operator, and $\mathcal{Q}(\mathbf x)=\{\mathbf q_1,\ldots,\mathbf q_T\}$ is the set of query representations extracted from the input text $\mathbf x$.
As a result, tokens with below-median surrogate error receive smaller $\lambda^{*}$ and rely more on $\bm{\pi}_{\mathrm{knn}}$, whereas tokens with above-median surrogate error receive larger $\lambda^{*}$ and fall back more heavily on $\bm{\pi}_{\mathrm{prx}}$.
This yields a simple and effective token-adaptive instantiation of \alg.

\begin{table*}[t]
\centering
\caption{Summary of benchmarks in our experiments.}
\label{tab:benchmark}
\scalebox{0.75}{%
\begin{tabular}{l l l r r l}
\toprule
\textbf{Benchmark} & \textbf{Domains (Topics)} & \textbf{Source LLMs} & \textbf{\#HWT} & \textbf{\#LGT} & \textbf{Notes} \\
\midrule

Mix8
& \makecell[l]{arXiv (academic) \\ XSum (news) \\ WritingPrompts (story) \\ PubMedQA (technical QA)}
& \makecell[l]{GPT-3.5, GPT-4, Claude-3 Sonnet, \\ Claude-3 Opus, Gemini-1.5 Pro, \\ GPT-4 Turbo, Claude-3.7 Sonnet, \\ Gemini-2.0 Flash}
& 900
& 3600
& \makecell[l]{Mix8 builds from the Glimpse~\cite{bao2025glimpse} and \\ DNA-DetectLLM~\cite{zhu2025dna} datasets, covering \\ 4 domains and 8 proprietary LLMs.} \\
\midrule
DetectRL
& \makecell[l]{arXiv (academic) \\ XSum (news) \\ WritingPrompts (story) \\ Yelp Reviews (social)}
& \makecell[l]{GPT-3.5 Turbo, \\ PaLM-2-Bison, \\ Claude-Instant, \\ Llama-2-70B}
& 4016
& 4016
& \makecell[l]{DetectRL~\cite{wu2024detectrl} includes prompt-based, \\ paraphrasing, and perturbation attacks \\ on the LGTs, covering 4 domains and 4 LLMs.} \\

\bottomrule
\end{tabular}%
}
\end{table*}

\section{Experiments}\label{sec:experiment}


\definecolor{bestcol}{RGB}{255,210,90}
\definecolor{secondcol}{RGB}{180,220,255}
\newcommand{\best}[1]{\cellcolor{bestcol}\textbf{#1}} 
\newcommand{\second}[1]{\cellcolor{secondcol}\underline{#1}}      

\DeclareRobustCommand{\bestcap}[1]{%
  \begingroup\setlength{\fboxsep}{1pt}\colorbox{bestcol}{\strut\textbf{#1}}\endgroup}
\DeclareRobustCommand{\secondcap}[1]{%
  \begingroup\setlength{\fboxsep}{1pt}\colorbox{secondcol}{\strut \underline{#1}}\endgroup}

\subsection{Settings} 
\label{sec:expsetup}

We evaluate \alg on the Mix8~\cite{bao2025glimpse, zhu2025dna} and DetectRL~\cite{wu2024detectrl} benchmarks, with Table~\ref{tab:benchmark} summarizing the benchmark details. All experiments are conducted in the black-box setting, where the proxy LLM used by the detector is different from the source LLM that generated the test samples. We report the AUROC and F1 score (higher the better), which are standard metrics in LLM-generated text detection~\cite{wu2024detectrl, bao2025glimpse}.

\begin{table*}[htbp!]
\centering
\caption{AUROC comparisons of black-box zero-shot LGT detection methods w/wo proxy alignment on Mix8. 
}
\label{tab:mix8}
\setlength{\tabcolsep}{5.25pt}
\scalebox{0.75}{%
\begin{threeparttable}
\begin{tabular}{llccccccccl}

\toprule

\textbf{Method} & \textbf{Proxy LLM} & \small  GPT-3.5 & \small GPT-4 & \small \makecell{Claude-3\\Sonnet} & \small \makecell{Claude-3\\Opus} & \small \makecell{Gemini-1.5 \\ Pro} & \small \makecell{GPT-4 \\ Turbo} & \small \makecell{Claude-3.7 \\ Sonnet} & \small  \makecell{Gemini-2.0 \\ Flash}   & \textsc{\textbf{Avg.}} \gainraw{Gain} \\

\specialrule{0.05em}{0.4ex}{0ex}

\rowcolor{gray!15} \multicolumn{11}{l}{\textit{\textbf{Zero-shot Detector (w/o Proxy Alignment)}}} \\ 
\rule{0pt}{10pt}Likelihood~\cite{ge2019gltr}            & GPT-Neo-2.7B                      & 0.907 & 0.769 & 0.866 & 0.903 & 0.742 & 0.676 & 0.762 & 0.839 & 0.808  \\
Entropy~\cite{la2008detecting}               & GPT-Neo-2.7B                      & 0.314 & 0.411 & 0.347 & 0.327 & 0.396 & 0.506 & 0.419 & 0.460 & 0.397  \\
Rank~\cite{ge2019gltr}                  & GPT-Neo-2.7B                      & 0.704 & 0.645 & 0.689 & 0.706 & 0.626 & 0.680 & 0.746 & 0.735 & 0.691  \\
LogRank~\cite{ge2019gltr}               & GPT-Neo-2.7B                      & 0.906 & 0.763 & 0.865 & 0.904 & 0.735 & 0.680 & 0.768 & 0.840 & 0.808  \\
DNA-GPT~\cite{yang2024dna}               & GPT-Neo-2.7B                      & 0.719 & 0.643 & 0.708 & 0.733 & 0.644 & 0.618 & 0.633 & 0.795 & 0.687  \\
DetectGPT~\cite{mitchell2023detectgpt}             & T5-11B\,+\,GPT-Neo-2.7B             & 0.783 & 0.614 & 0.797 & 0.778 & 0.741 & 0.455 & 0.542 & 0.594 & 0.663  \\
$\spadesuit$ Fast-DetectGPT~\cite{bao2024fast} & GPT-J-6B\,+\,GPT-Neo-2.7B & 0.949 & 0.900 & 0.926 & 0.947 & 0.807 & 0.854 & 0.810 & 0.955 & 0.893  \\
$\diamondsuit$ Binoculars~\cite{hans2024spotting} & Falcon-7B\,+\,Falcon-7B-Instruct & \second{0.994} & \second{0.981} & \second{0.981} & \second{0.986} & \second{0.920} & \second{0.966} & \second{0.971} & \second{0.990} & \second{0.974}  \\

\specialrule{0.05em}{0.4ex}{0ex}

\rowcolor{gray!15} \multicolumn{11}{l}{\textit{\textbf{Fast-DetectGPT (w/ Proxy Alignment)}}} \\

\rule{0pt}{10pt}$\spadesuit$ + DALD~\cite{zeng2024dald} & \makecell[l]{Llama-2-7B\,+\,Llama-2-7B$^{\star}$} & 0.977 & 0.899 & 0.966 & 0.965 & 0.889 & 0.811 & 0.899 & 0.981 & 0.923 \gain{3.36} \\

$\spadesuit$ + Glimpse-Geometric~\cite{bao2025glimpse} & GPT-4 (API) & 0.974 & 0.965 & 0.962 & 0.982 & 0.895 & - & - & - & \ \ \;-  \\

$\spadesuit$ + Glimpse-Zipfian~\cite{bao2025glimpse} & GPT-4 (API) & 0.977 & 0.972 & 0.961 & 0.979 & 0.899 & - & - & - & \ \ \;- \\

$\spadesuit$ + Glimpse-MLP~\cite{bao2025glimpse} & GPT-4 (API) & 0.977 & 0.971 & 0.963 & 0.981 & 0.900 & - & - & - & \ \ \;- \\

$\spadesuit$ + Glimpse-Geometric~\cite{bao2025glimpse}  & DaVinci-002 (API) & 0.976 & 0.913 & 0.961 & 0.974 & 0.860 & 0.845 & 0.924 & 0.987 & 0.930 \gain{4.14}  \\




$\spadesuit$ + \alg (Ours) & \makecell[l]{Falcon-7B\,+\,Falcon-7B-Instruct$^{\star}$} & 0.976 & 0.958 & 0.961 & 0.967 & 0.890 & 0.948 & 0.946 & 0.987 & 0.954 \gain{6.83} \\

\rowcolor{gray!15} \multicolumn{11}{l}{\textit{\textbf{Binoculars (w/ Proxy Alignment)}}} \\

$\diamondsuit$ + DALD~\cite{zeng2024dald} & \makecell[l]{Llama-2-7B\,+\,Llama-2-7B$^{\star}$} & 0.960 & 0.868 & 0.940 & 0.948 & 0.810 & 0.689 & 0.773 & 0.937 & 0.866 \loss{11.09}\\

$\diamondsuit$ + \alg (Ours) & \makecell[l]{Falcon-7B\,+\,Falcon-7B-Instruct$^{\star}$} & \best{0.996} & \best{0.994} & \best{0.992} & \best{0.991} & \best{0.975} & \best{0.989} & \best{0.988} & \best{0.993} & \best{0.990} \gain{1.64} \\

\bottomrule
\end{tabular}%
\begin{tablenotes}[para,flushleft]
\textbf{\textit{Notes:}} 
Detectors using two proxy LLMs are denoted as ``reference LLM\,+\,scoring LLM''. $\spadesuit$: Fast-DetectGPT; $\diamondsuit$: Binoculars; $\star$: the aligned LLM; ``–'': unavailable results. The best and second-best results are highlighted in \bestcap{yellow} and \secondcap{blue}.
\end{tablenotes}
\end{threeparttable}%
}
\end{table*}

\subsubsection{Mix8 Benchmarking} 
We construct the Mix8 benchmark by combining test sets from Glimpse~\cite{bao2025glimpse} and DNA-DetectLLM~\cite{zhu2025dna}, resulting in a unified suite that covers eight latest commercial LLMs. 
The Glimpse test set covers three domains, including XSum~\cite{narayan2018don}, WritingPrompts~\cite{fan2018hierarchical}, and PubMedQA~\cite{jin2019pubmedqa}. 
For each domain, we randomly sample 150 HWTs and generate LGTs by prompting each source model with the 30-token prefix of the same texts. 
The Glimpse test set includes five source LLMs, including GPT-3.5, GPT-4~\cite{achiam2023gpt}, Claude-3 Sonnet/Opus~\cite{claude3}, and Gemini-1.5 Pro~\cite{team2024gemini}. 
DNA-DetectLLM complements this setup with 150 sampled HWTs spanning XSum~\cite{narayan2018don}, WritingPrompts~\cite{fan2018hierarchical}, and arXiv~\cite{paul2021arXiv}, paired with three more recent proprietary models, GPT-4 Turbo~\cite{gpt4turbo}, Gemini-2.0 Flash~\cite{gemini2}, and Claude-3.7 Sonnet~\cite{claude37}.

\subsubsection{DetectRL Benchmarking} 
DetectRL~\cite{wu2024detectrl} covers 4 domains, including arXiv~\cite{paul2021arXiv}, XSum~\cite{narayan2018don}, WritingPrompts~\cite{fan2018hierarchical} and Yelp Reviews~\cite{zhang2015character}. 
LGT is generated by GPT-3.5 Turbo~\cite{gpt35}, PaLM-2-Bison~\cite{anil2023palm}, Claude Instant 1.2~\cite{claudei}, and Llama-2-70B~\cite{touvron2023llama2}. Crucially, DetectRL includes a large number of attack-augmented test samples beyond direct prompting, including prompt-based, paraphrase, and perturbation attacks. 

\subsubsection{Competitors} 
We compare \alg against 11 zero-shot detectors, including Likelihood~\cite{ge2019gltr}, Entropy~\cite{la2008detecting}, Rank~\cite{ge2019gltr}, LogRank~\cite{ge2019gltr}, LRR~\cite{su2023detectllm}, NPR~\cite{su2023detectllm}, Revise-Detect~\cite{zhu2023beat}, DNA-GPT~\cite{yang2024dna}, DetectGPT~\cite{mitchell2023detectgpt}, Fast-DetectGPT~\cite{bao2024fast}, and Binoculars~\cite{hans2024spotting}. Each detector is evaluated with its default proxy LLM, including GPT-Neo-2.7B~\cite{gptneo}, GPT-J-6B~\cite{gptj}, T5-11B~\cite{t5}, Llama-2-7B~\cite{touvron2023llama2}, Falcon-7B~\cite{falcon}, and Falcon-7B-Instruct~\cite{falconi} (cf. Table~\ref{tab:mix8}), as well as 2 proxy alignment methods, DALD~\cite{zeng2024dald} (with their open-source aligned Llama-2-7B~\cite{dald_llama}) and Glimpse~\cite{bao2025glimpse} with DaVinci-002 API~\cite{davinci002}.

\subsubsection{Implementation Details} 
For Mix8, we build the datastore with the GPT-4 responses in WildChat~\cite{zhao2024wildchat}, matching the same training set used by DALD~\cite{zeng2024dald} for a fair comparison. 
For DetectRL, we use the training split to construct the datastore. For \algm, we use NV-Embed-v2~\cite{lee2025nvembed} to obtain sentence embeddings for routing. 
All experiments are run on a single NVIDIA RTX A6000 GPU. 

\subsection{Results on Mix8}
\label{sec:mix8}

\newcommand{\method}[1]{\makebox[3cm][l]{#1}}

\begin{table*}[t]
\centering

\caption{In-domain evaluation (AUROC) on DetectRL under two complementary settings: Multi-Topic and Multi-LLM.}

\label{tab:detectrl_in}
\setlength{\tabcolsep}{9pt}
\scalebox{0.75}{%
\begin{threeparttable}
\begin{tabular}{lcccclccccl}

\toprule

& \multicolumn{5}{c}{\textbf{In-Domain Multi-Topic}} & \multicolumn{5}{c}{\textbf{In-Domain Multi-LLM}} \\

\cmidrule(lr){2-6}  \cmidrule(lr){7-11}

\textbf{Method} & arXiv & XSum & Writing & Review & \textsc{\textbf{Avg.}} \gainraw{Gain} & GPT-3.5 & Claude & PaLM-2 & Llama-2 & \textsc{\textbf{Avg.}} \gainraw{Gain} \\
\specialrule{0.05em}{0.4ex}{0ex}
\rowcolor{gray!15} \multicolumn{11}{l}{\textit{\textbf{Zero-shot Detector (w/o Proxy Alignment)}}} \\ 
\rule{0pt}{10pt}$\heartsuit$ Likelihood~\cite{ge2019gltr}      & 0.988 & 0.457 & 0.680 & 0.758 & 0.632 & 0.629 & 0.433 & 0.700 & 0.757 & 0.625  \\
Entropy~\cite{la2008detecting}             & 0.484 & 0.678 & 0.391 & 0.288 & 0.465 & 0.468 & 0.523 & 0.453 & 0.435 & 0.470  \\
Rank~\cite{ge2019gltr}                & 0.572 & 0.369 & 0.563 & 0.551 & 0.511 & 0.522 & 0.417 & 0.504 & 0.571 & 0.503  \\
LogRank~\cite{ge2019gltr}            & 0.670 & 0.467 & 0.676 & 0.764 & 0.644 & 0.628 & 0.433 & 0.709 & 0.780 & 0.638  \\
LRR~\cite{su2023detectllm}                 & 0.705 & 0.501 & 0.647 & 0.766 & 0.655 & 0.616 & 0.433 & 0.712 & 0.837 & 0.644  \\
NPR~\cite{su2023detectllm}                 & 0.539 & 0.346 & 0.550 & 0.501 & 0.489 & 0.503 & 0.416 & 0.446 & 0.525 & 0.478  \\
DetectGPT~\cite{mitchell2023detectgpt}           & 0.222 & 0.122 & 0.590 & 0.444 & 0.344 & 0.435 & 0.329 & 0.267 & 0.367 & 0.344  \\
Revise-Detect~\cite{zhu2023beat}      & 0.704 & 0.503 & 0.732 & 0.750 & 0.678 & 0.701 & 0.499 & 0.698 & 0.757 & 0.669  \\
DNA-GPT~\cite{yang2024dna}             & 0.674 & 0.642 & 0.690 & 0.782 & 0.697 & 0.619 & 0.489 & 0.715 & 0.752 & 0.649 \\
$\spadesuit$ Fast-DetectGPT~\cite{bao2024fast}      & 0.437 & 0.392 & 0.742 & 0.770 & 0.580 & 0.656 & 0.300 & 0.660 & 0.768 & 0.596  \\
$\diamondsuit$ Binoculars~\cite{hans2024spotting}          & 0.840 & 0.774 & 0.944 & 0.900 & 0.870 & 0.881 & 0.552 & 0.933 & 0.966 & 0.833 \\


\rowcolor{gray!15} \multicolumn{11}{l}{\textit{\textbf{Zero-shot Detector (w/ Proxy Alignment)}}} \\ 
\rule{0pt}{10pt}$\heartsuit$ + \alg (Ours) & 0.986 &  0.966 &  0.982 &  0.997 &  0.983 \gain{55.54} &  0.969 &  0.958 &  0.969 &  0.991 & 0.972  \gain{55.52} \\

$\spadesuit$ + \alg (Ours) & \second{0.998} & \second{0.987} & \second{0.988} & \second{0.997} & \second{0.992}  \gain{71.03}  & \second{0.983} & \second{0.988} &  \second{0.990} &  \second{0.994} & \second{0.989}   \gain{65.94}  \\

$\diamondsuit$ + \alg (Ours) & \best{1.000} &  \best{0.997} &  \best{0.998} &  \best{1.000} &  \best{0.999}  \gain{14.83}  & \best{0.997} &  \best{0.999} &  \best{0.998} &  \best{0.999} &  \best{0.998}  \gain{19.81}  \\
\bottomrule

\end{tabular}%
\begin{tablenotes}[para,flushleft]
\textbf{\textit{Notes:}} 
$\heartsuit$: Likelihood; $\spadesuit$: Fast-DetectGPT; $\diamondsuit$: Binoculars. The best and second-best results are highlighted in \bestcap{yellow} and \secondcap{blue}.
\end{tablenotes}
\end{threeparttable}%
}
\end{table*}

\begin{table*}[t]
\centering
\caption{Out-of-domain evaluation (F1 score) on DetectRL under Multi-LLM and Multi-Topic settings.
}

\label{tab:detectrl_out}
\setlength{\tabcolsep}{7.5pt}
\scalebox{0.75}{%
\begin{threeparttable}
\begin{tabular}{llccccclccccc}

\toprule

& \multicolumn{6}{c}{\textbf{Out-of-Domain Multi-LLM}} 
& \multicolumn{6}{c}{\textbf{Out-of-Domain Multi-Topic}} \\

\cmidrule(r){2-7}  \cmidrule(r){8-13}

\textbf{Method} 
& \backslashbox{Proxy \kern-1.2em}{\kern-1.2em Source} & GPT-3.5 & Claude & PaLM-2 & Llama-2 & \textsc{\textbf{Avg.}} 
& \backslashbox{Proxy \kern-1.2em}{\kern-1.2em Source} & arXiv & XSum & Writing & Review & \textsc{\textbf{Avg.}} \\

\midrule


\multirow{4}{*}{$\spadesuit$ Fast-DetectGPT~\cite{bao2024fast}}
& GPT-3.5 & - & 0.130  & 0.596  & 0.699  &  & arXiv & - & 0.237  & 0.597  & 0.602  & \\
& Claude & 0.002  & - & 0.000  & 0.012  &  & XSum & 0.284  & - & 0.630  & 0.631  &  \\
& PaLM-2 & 0.558  & 0.082  & - & 0.684  &  & Writing & 0.348  & 0.336  & - & 0.683  &  \\
& Llama-2 & 0.563  & 0.087  & 0.577  & - & \multirow{-4}{*}{0.332} & Review & 0.407  & 0.377  & 0.683  & - & \multirow{-4}{*}{0.484} \\

\midrule

\multirow{4}{*}{$\diamondsuit$ Binoculars~\cite{hans2024spotting}}
& GPT-3.5 & - & 0.397  & 0.880  & 0.917  & & arXiv & - & 0.720  & 0.795  & 0.836  & \\
& Claude & 0.822  & - & 0.879  & 0.912  &  & XSum & 0.763  & - & 0.797  & 0.832  & \\
& PaLM-2 & 0.772 & 0.303 & - & 0.910 &  & Writing & 0.736 & 0.697 & - & 0.818 & \\
& Llama-2 & 0.814 & 0.355 & 0.886 & - & \multirow{-4}{*}{0.737} & Review & 0.777 & 0.720 & 0.794 & - & \multirow{-4}{*}{0.774} \\

\midrule

\multirow{4}{*}{\shortstack[l]{$\spadesuit$ + \alg (Ours)}}

& GPT-3.5 & - & 0.453  & 0.818  & 0.936  & \cellcolor{secondcol} 

& arXiv & - & 0.737  & 0.813  & 0.865   & \cellcolor{secondcol} \\

& Claude & 0.902  & - & 0.808  & 0.918   & \cellcolor{secondcol} & XSum & 0.751  & - & 0.849  & 0.876   & \cellcolor{secondcol} \\

& PaLM-2 & 0.908  & 0.471  & - & 0.923   & \multirow{-3}{*}{\cellcolor{secondcol}\underline{0.783}} 

& Writing & 0.731  & 0.741  & - & 0.868   & \multirow{-3}{*}{\cellcolor{secondcol}\underline{0.795}} \\

& Llama-2 & 0.917  & 0.515  & 0.828  & - & \multirow{-2}{*}{\cellcolor{secondcol}\gain{135.84}} 

& Review & 0.728  & 0.733  & 0.845  & - & \multirow{-2}{*}{\cellcolor{secondcol}\gain{64.26}} \\

\midrule

\multirow{4}{*}{\shortstack[l]{$\diamondsuit$ + \alg (Ours)}}

& GPT-3.5 &- & 0.679  & 0.966  & 0.981   & \cellcolor{bestcol} & arXiv & - & 0.918  & 0.878  & 0.907  & \cellcolor{bestcol} \\

& Claude & 0.954  & - & 0.962  & 0.979  & \cellcolor{bestcol} & XSum & 0.902  & - & 0.907  & 0.932   & \cellcolor{bestcol} \\

& PaLM-2 & 0.965  & 0.712  & - & 0.982   & \multirow{-3}{*}{\cellcolor{bestcol}\textbf{0.906}} & 

Writing & 0.882  & 0.939  & - & 0.931  & \multirow{-3}{*}{\cellcolor{bestcol}\textbf{0.909}} \\

& Llama-2 & 0.971  & 0.759  & 0.965  & - & \multirow{-2}{*}{\cellcolor{bestcol}\gain{22.93}} 

& Review & 0.877  & 0.935  & 0.901  & - & \multirow{-2}{*}{\cellcolor{bestcol}\gain{17.44}} \\

\midrule

$\spadesuit$ + \algm (Ours) & \multicolumn{6}{c}{N/A} & w/o src & 0.738  & 0.740  & 0.844  & 0.869  & 0.798  \\

\midrule

$\diamondsuit$ + \algm (Ours) & \multicolumn{6}{c}{N/A} & w/o src & 0.880  & 0.933  & 0.901  & 0.931  & 0.911  \\

\bottomrule
\end{tabular}%
\begin{tablenotes}[para,flushleft]
\textbf{\textit{Notes:}} 
$\spadesuit$: Fast-DetectGPT; $\diamondsuit$: Binoculars. For \alg, we construct the datastore using the training set of a single LLM/topic (row label) and evaluate it on the remaining LLMs/topics (columns). ``-'' denotes the in-domain entries which are omitted. The best and second-best results are highlighted in \bestcap{yellow} and \secondcap{blue}.
\end{tablenotes}
\end{threeparttable}%
}
\end{table*}

Table~\ref{tab:mix8} reports AUROC on the Mix8 benchmark, where results for each source LLM are averaged over three topic domains. Overall, our strongest instantiation (Binoculars detector $\diamondsuit$ with proxy LLMs Falcon-7B + Falcon-7B-Instruct$^{\star}$, ${\star}$ denotes the aligned LLM) achieves the best performance across all source LLMs, reaching an average AUROC of 0.990. Compared with the default Binoculars setting (0.974), the relative improvement is 1.64\% under the same proxy LLM pair. In contrast, combining Binoculars with DALD yields 0.866 average AUROC, indicating a relative decrease of 11.09\% versus the original Binoculars detector.

Among zero-shot methods, classic methods such as Rank and Entropy exhibit limited separability on modern closed models, whereas the latest detectors, such as Fast-DetectGPT and Binoculars are substantially stronger. Therefore, we further instantiate Fast-DetectGPT $\spadesuit$ and Binoculars $\diamondsuit$ with aligned proxy LLMs to compare different proxy-alignment strategies. Notably, Glimpse relies on top-$k$ probabilities returned by the API to approximate the full next-token distributions which can be fragile in practice. For example, the GPT-4 API no longer returns top-$k$ probabilities~\cite{glimpse_issue}, so we mark the corresponding entries as ``-''. For Fast-DetectGPT, all alignment strategies improve over its original zero-shot baseline of 0.893. DALD reaches 0.923 with a gain of 3.36\%. Glimpse with DaVinci-002 reaches 0.930 with a gain of 4.14\%. \alg achieves the best AUROC of 0.954 with a gain of 6.83\%. The results demonstrate that \alg provides a more reliable proxy alignment mechanism that leverages corpus retrieval to strengthen black-box LGT detection across a broad set of proprietary LLMs.

\subsection{Results on DetectRL}
\label{sec:detectrl}

\subsubsection{In-domain Evaluation}
Table~\ref{tab:detectrl_in} reports the in-domain results on DetectRL under the Multi-Topic and Multi-LLM settings, in which \alg uses the in-domain training set to build the datastore. We instantiate \alg on top of multiple zero-shot detectors and observe consistent gains in both settings. Our strongest instantiation, Binoculars $\diamondsuit$ + \alg, reaches near-ceiling AUROC at 0.999 for Multi-Topic and 0.998 for Multi-LLM. The improvement is not limited to a single baseline. Likelihood improves by about 56\%, Fast-DetectGPT improves by roughly 66--71\%, and even Binoculars still gains around 15--20\%. In addition, the original Binoculars setting shows notable cross-LLM variance in Multi-LLM, dropping to 0.552 on Claude, while \alg remains consistently strong across different source LLMs. These results suggest that aligning the proxy distribution with a retrieval-augmented distribution yields a reliable detection signal that remains stable across different topics and source-LLMs, especially with attack-augmented test samples.

\subsubsection{Out-of-domain Evaluation}
Beyond the in-domain evaluation, we further stress-test the detectors under the out-of-domain setting on DetectRL, with F1 score results reported in Table~\ref{tab:detectrl_out}. For \alg, we use a single unified threshold across all LLMs and topics. We find that \alg remains effective under distribution shift. Our strongest variant, $\diamondsuit$ + \alg, reaches a mean F1 score of about 0.91 in both Multi-LLM and Multi-Topic. Relative to the original Binoculars baseline, this translates to a clear gain of roughly 17--23\%. The effect is even more pronounced for Fast-DetectGPT, where alignment boosts F1 score from around 0.3--0.4 to about 0.79, especially in Multi-LLM. Overall, these results show that \alg yields a substantially more stable detection signal under combined topic shifts, source-LLM shifts, and attack perturbations.

\begin{figure}[t]
\vspace{-10pt}
\centering

\subfloat[Source Topic Included]{%
    \includegraphics[width=0.5\linewidth]{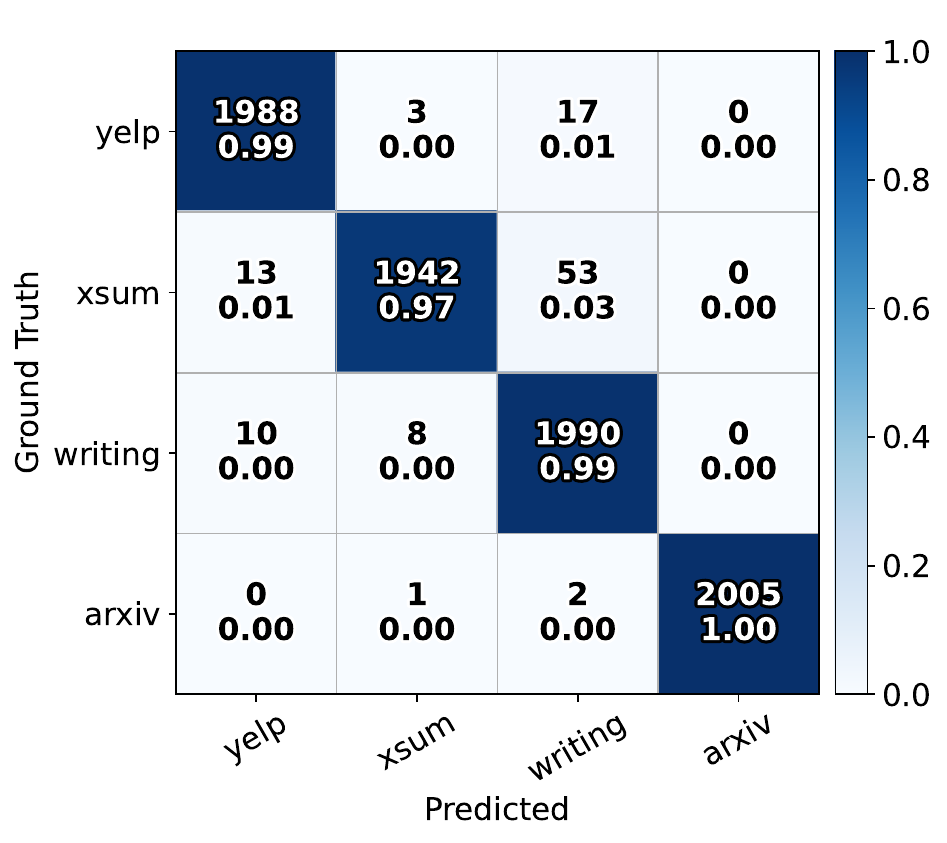}%
}
\subfloat[Source Topic Excluded]{%
    \includegraphics[width=0.5\linewidth]{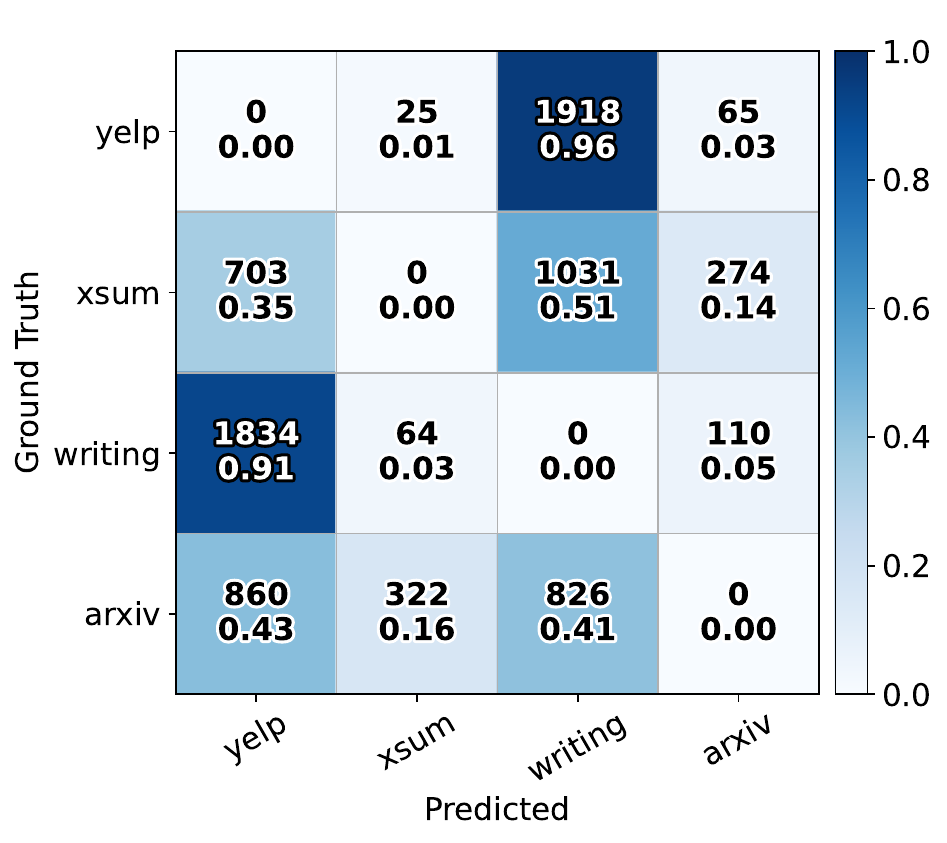}%
}

\caption{\textbf{Confusion matrix of the topic MoE router predictions.} 
Each cell reports the count and row-normalized proportion (rows sum to 1), and the diagonal entries give the per-topic accuracy.}
\label{fig:moe_router}
\end{figure}

\subsubsection{\algm Results} 
Additionally, we evaluate the effectiveness of our \algm design for cross-topic generalization.
For each source topic, we build three topic-specific \alg\!\!s using the training sets of the remaining topics as the retrieval corpora, and use the topic router (Sec.~\ref{sec:moknnproxy}) to dispatch each test sample to the most suitable topic-specific \alg. 
As reported in the last two rows in Table~\ref{tab:detectrl_out}, \algm achieves the average F1 score of 0.798 for $\spadesuit$ + \algm and 0.911 for $\diamondsuit$ + \algm, showing that routing across domain-specific proxies can further improve OOD robustness. 
To assess the accuracy of the MoE router, we show the confusion matrices in Fig.~\ref{fig:moe_router} under two settings: including and excluding the source topic. 
Using a sentence-embedding model with a \knn router, we achieve over 0.97 routing accuracy across the four topics. When the source topic is excluded, the router assigns each test sample to its most similar topic, and the results align with the topical readability and lexical distribution of DetectRL.

\begin{figure}[t]
\vspace{-10pt}
\centering

\subfloat[Temperature $\tau$]{%
    \includegraphics[width=0.5\linewidth]{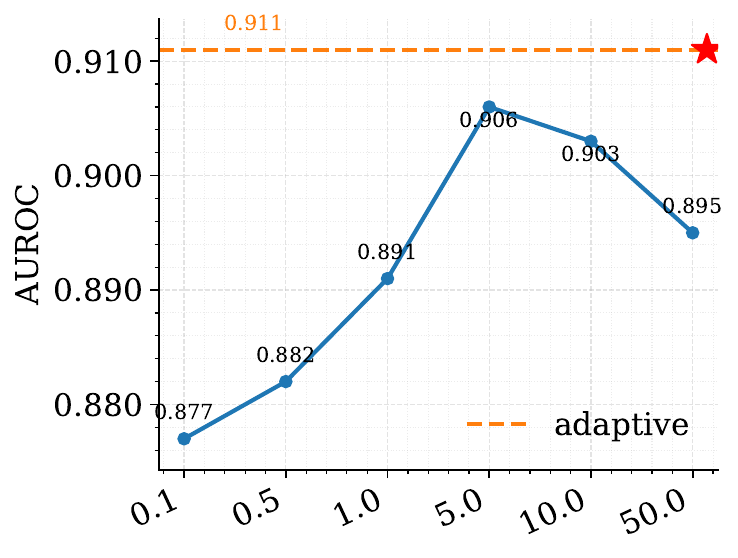}%
}
\subfloat[Neighbor Size $k$]{%
    \includegraphics[width=0.5\linewidth]{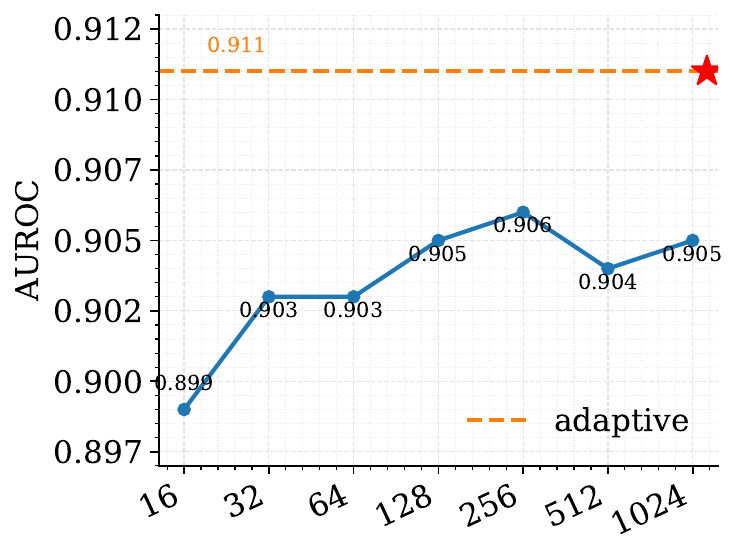}%
}
\\ 
\subfloat[Interpolation weight $\lambda$]{%
    \includegraphics[width=0.5\linewidth]{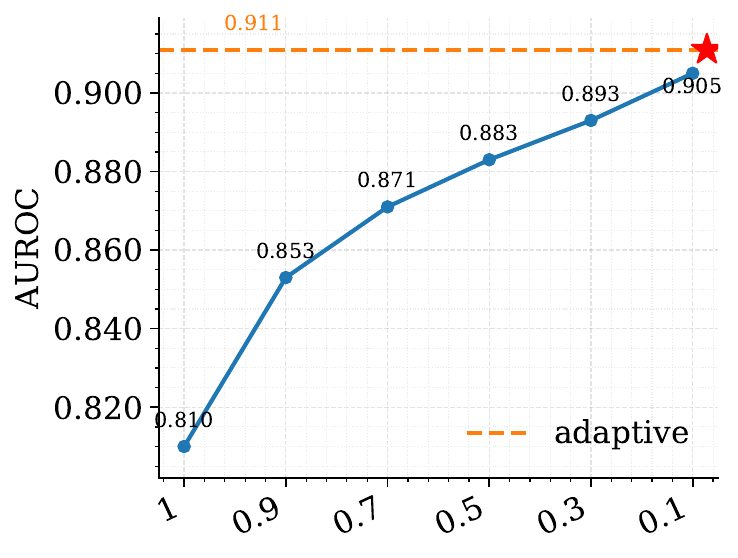}%
}
\subfloat[Clip Value $\gamma$]{%
    \includegraphics[width=0.5\linewidth]{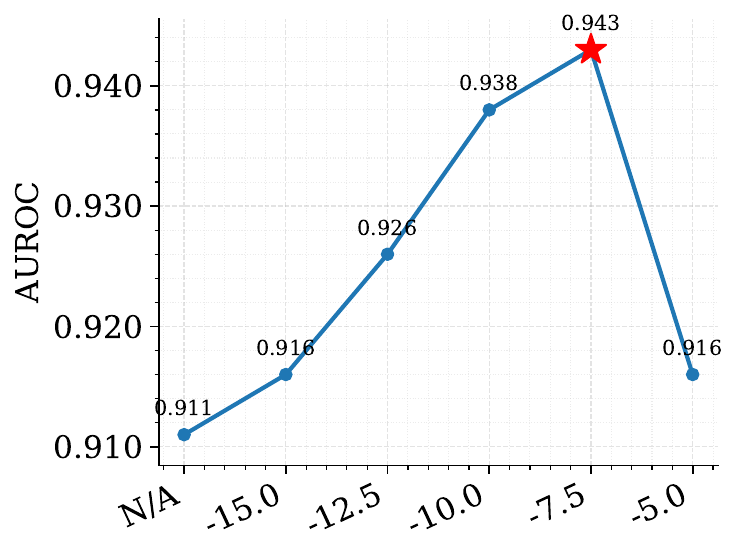}%
}
\\
\subfloat[Corpus Size $N$]{%
    \includegraphics[width=0.5\linewidth]{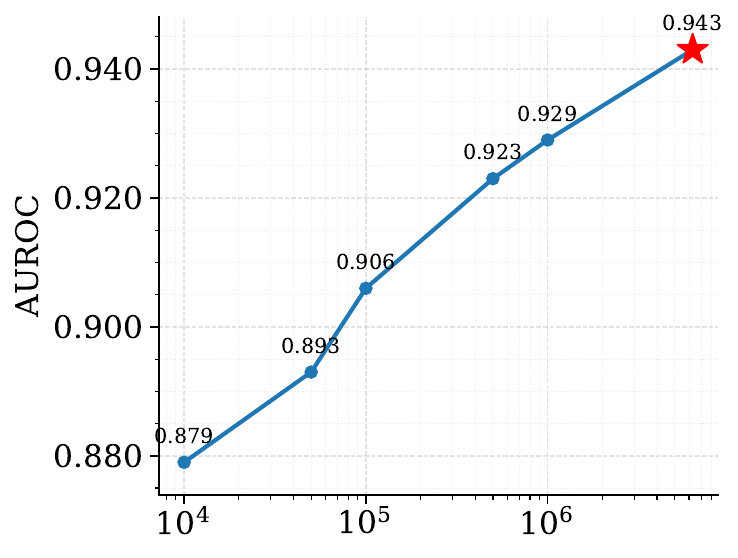}%
}
\caption{\textbf{Ablation study of \alg hyperparameters.} 
Each panel shows the AUROC (↑) performance and $\star$ marks the best-performing setting. 
“adaptive” indicates using the corresponding adaptive strategy in Sec.~\ref{sec:adapt_param}.}
\label{fig:ablation}
\end{figure}
 
\subsection{Ablation Studies} 
\label{sec:ablation}
 
\subsubsection{\alg Hyperparameters} 
We analyze the sensitivity of \alg to 5 key hyperparameters in Fig.~\ref{fig:ablation}.
\textbf{(1)~Temperature:} 
The temperature parameter $\tau$ controls how sharply retrieval weights concentrate on nearest neighbors. With fixed values, performance improves from 0.877 at $\tau=0.1$ to a peak of 0.906 at $\tau=5$, then drops to 0.895 at larger values $\tau=50$. The adaptive-temperature variant further improves performance to 0.911.
\textbf{(2)~Neighbor Size $k$:} 
We vary $k$ from 16 to 1024 and evaluate the adaptive-$k$ strategy. Fixed-$k$ performance is relatively stable in the 0.899--0.906 range, with the best fixed point at $k=256$. The adaptive-$k$ variant achieves 0.911 and outperforms all fixed settings.
\textbf{(3)~Interpolation Weight $\lambda$:} 
The interpolation weight controls how much the retrieval distribution contributes to the final distribution. Recall that a smaller $\lambda$ means a larger contribution from the retrieval distribution. As $\lambda$ decreases from 1 to 0.1, AUROC improves consistently from 0.810 to 0.905, and adaptive $\lambda$ reaches the best value of 0.911. 
\textbf{(4)~Clip Value $\gamma$:} 
We study clipping to suppress extreme low-confidence token effects. Starting from 0.911 without clipping, AUROC increases as clipping becomes moderately stronger and reaches the best value 0.943 at $\gamma=-7.5$. Overly aggressive clipping then hurts performance, dropping to 0.916 at $\gamma=-5.0$.
\textbf{(5)~Corpus Size $N$:} 
We study how the size of the domain corpus used to build the datastore affects performance. Larger corpora consistently improve AUROC, from 0.879 at 1\,\texttimes\,10$^\text{4}$ to 0.943 with the full corpus 6\,\texttimes\,10$^\text{6}$, confirming that richer domain coverage provides more reliable neighbor evidence and strengthens proxy alignment.


\begin{table}[t]
\caption{Ablation studies of components in \alg.
}
\label{tab:ablation}
\centering
\setlength{\tabcolsep}{6pt}
\scalebox{0.75}{%
\begin{threeparttable}
\begin{tabular}{lcccccr}

\toprule

\textbf{Setting}
& \small GPT-3.5 &\small GPT-4 &\small Claude-S &\small Claude-O &\small Gemini & \textsc{\textbf{Avg.}} \\

\specialrule{0.05em}{0.4ex}{0ex}

\rowcolor{gray!15}
\multicolumn{7}{l}{\textit{\textbf{Corpus Source LLM}}} \\
\rule{0pt}{10pt}Claude
& 0.943 & 0.867 & 0.916 & 0.937 & 0.797 & 0.897 \\
PaLM
& 0.955 & 0.908 & 0.916 & 0.940 & 0.825 & 0.913 \\
Llama
& 0.961 & 0.901 & 0.921 & 0.948 & 0.815 & 0.914 \\
GPT-3.5
& 0.968 & 0.932 & 0.937 & 0.952 & 0.857 & 0.933 \\
\rowcolor{yellow!15} GPT-4
& 0.984 & 0.954 & 0.966 & 0.975 & 0.888 & \best{0.953}\\

\specialrule{0.05em}{0.4ex}{0ex}

\rowcolor{gray!15}
\multicolumn{7}{l}{\textit{\textbf{Prompt Template}}} \\
\rule{0pt}{10pt}Prompt 0
& 0.984 & 0.954 & 0.966 & 0.975 & 0.888 & 0.953 \\
Prompt 1
& 0.984 & 0.961 & 0.967 & 0.977 & 0.896 & 0.957 \\
\rowcolor{yellow!15} Prompt 2
& 0.984 & 0.965 & 0.972 & 0.979 & 0.905 & \best{0.961} \\
Prompt 3
& 0.983 & 0.963 & 0.969 & 0.975 & 0.903 & 0.959 \\
Prompt 4
& 0.983 & 0.963 & 0.969 & 0.975 & 0.905 & 0.959 \\

\specialrule{0.05em}{0.4ex}{0ex}

\rowcolor{gray!15}
\multicolumn{7}{l}{\textit{\textbf{Base Proxy LLM}}} \\
\rule{0pt}{10pt}GPT-Neo-2.7B & 0.979 & 0.941 & 0.947 & 0.963 & 0.873 & 0.941 \\
\rowcolor{yellow!15} Falcon-7B-Instruct & 0.984 & 0.954 & 0.966 & 0.975 & 0.888 & \best{0.953} \\
Falcon-7B & 0.920 & 0.745 & 0.882 & 0.916 & 0.715 & 0.836 \\
Llama-2-7B & 0.970 & 0.894 & 0.944 & 0.952 & 0.858 & 0.924 \\
Llama-3.1-8B & 0.967 & 0.854 & 0.940 & 0.959 & 0.846 & 0.913 \\

\bottomrule

\end{tabular}%
\begin{tablenotes}[para,flushleft]
\textbf{\textit{Notes:}} 
Claude-S: Claude-3 Sonnet; Claude-O: Claude-3 Opus; Gemini: Gemini-1.5 Pro.
The best results are highlighted in \bestcap{yellow}. 
\end{tablenotes}
\end{threeparttable}%
}
\end{table}

\begin{table}[t!]
\centering
\caption{Illustration of the prompt templates. 
}
\label{tab:prompts}
\scalebox{0.75}{%
\begin{tabular}{lp{8.75cm}}
\toprule
\textbf{Prompt ID} & \textbf{Content} \\
\midrule
Prompt 0 & (Empty) \\
\midrule
Prompt A & \textit{You serve as a valuable aide, capable of generating clear and persuasive pieces of writing given a certain context. Now, assume the role of an author and strive to finalize this article.} \\
\midrule
Prompt B &  \textit{I operate as an entity utilizing GPT as the foundational large language model. I function in the capacity of a writer, authoring articles on a daily basis. Presented below is an example of an article I have crafted.} \\
\midrule
Prompt 1 & [Prompt A]\\
\midrule
Prompt 2 & [Prompt A] [Prompt B] \\
\midrule
Prompt 3 & \textit{System:} [Prompt A] \quad \textit{Assistant:} [Prompt B] \\
\midrule
Prompt 4 & \textit{Assistant:} [Prompt A] \quad \textit{User:} [Prompt B] \\
\bottomrule
\end{tabular}%
}
\end{table}

\subsubsection{Proxy-related Ablation} 
Table~\ref{tab:ablation} studies three proxy-related factors in \alg. Following Glimpse, we apply prompt concatenation before test samples, and the prompt templates are listed in Table~\ref{tab:prompts}. We report the detection AUROC for five black-box target LLMs, along with the average AUROC. \textbf{(1)~Corpus Source LLM:} The source LLM used to build the retrieval corpus strongly affects performance. As the corpus source becomes stronger, the average AUROC rises from 0.897 (Claude) to 0.953 (GPT-4), showing that higher-quality source-reflective corpora provide more effective alignment signals. \textbf{(2)~Prompt Template:} Prompt formatting also contributes. Prompt 2 (in Table~\ref{tab:prompts}) achieves the best average AUROC of 0.961, while the other prompts remain close in the 0.953--0.959 range, indicating that concatenating role-aware context provides a consistent but moderate gain. \textbf{(3)~Base Proxy LLM:} The base proxy choice has a clear impact. Falcon-7B-Instruct performs best with 0.953 average AUROC, followed by GPT-Neo-2.7B at 0.941. Other candidates are weaker, including Llama-2-7B at 0.924, Llama-3.1-8B at 0.913, and Falcon-7B at 0.836.

\begin{figure}[t]
\centering
\includegraphics[width=\linewidth]{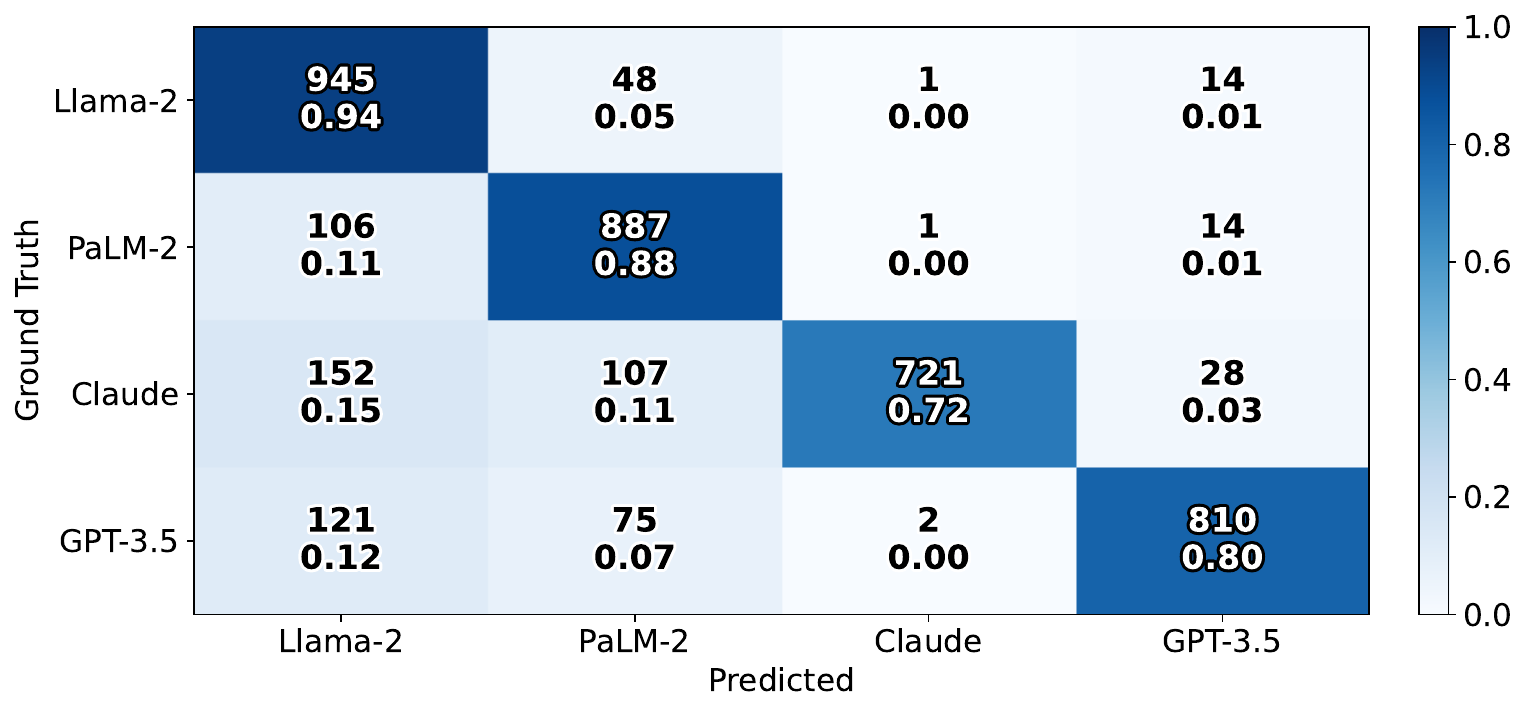}
\vspace*{-5mm}
\caption{\textbf{Confusion matrix of LLM source attribution. }
Each cell reports the count and row-normalized proportion (rows sum to 1), and the diagonal entries give the per-LLM accuracy.}
\label{fig:llm_attri}
\end{figure}

\subsection{LLM Source Attribution}
\label{sec:llmattri}

We further show that \alg is not limited to LGT detection, but can also be used for closed-set LLM source attribution. Given a set of source LLMs, we construct a distinct \alg for each source LLM using the source-LLM-generated corpus. 
At inference time, we compute the log-likelihood of the test samples under every \alg and predict the source as the model with the highest likelihood. 
To evaluate attribution accuracy, we conduct experiments on the DetectRL multi-LLM split, which contains 4 LLMs, including GPT-3.5 Turbo, PaLM-2-Bison, Claude-Instant, and Llama-2-70B. 
For each LLM, we use the training split to build the datastore, and evaluate on the test split with the confusion matrix shown in Fig.~\ref{fig:llm_attri}. \alg delivers strong closed-set source attribution, achieving an average accuracy of 0.84 across the four LLMs. In particular, it attributes Llama-2 and PaLM-2 reliably with accuracies of 0.94 and 0.88, respectively, while remaining competitive on GPT-3.5 and Claude. Specifically, most mistakes come from confusing Claude with Llama-2 or PaLM-2, which is consistent with DetectRL observations that Claude/PaLM/Llama generations have similar readability statistics and can be harder to separate. 


\begin{figure*}[tb]
\centering

\subfloat[Human-written text, w/o \alg, Avg. log-likelihood = -2.660]{%
    \includegraphics[width=\linewidth]{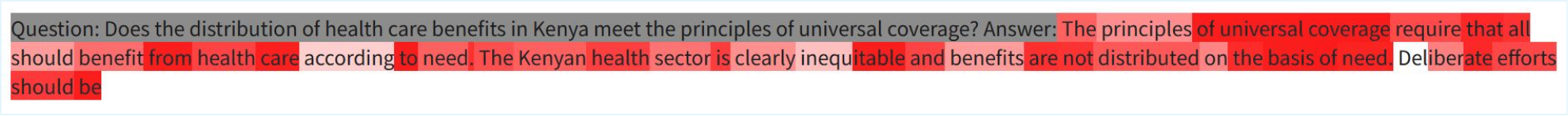}%
}

\subfloat[GPT-4-generated text, w/o \alg, Avg. log-likelihood = -2.705
]{%
    \includegraphics[width=\linewidth]{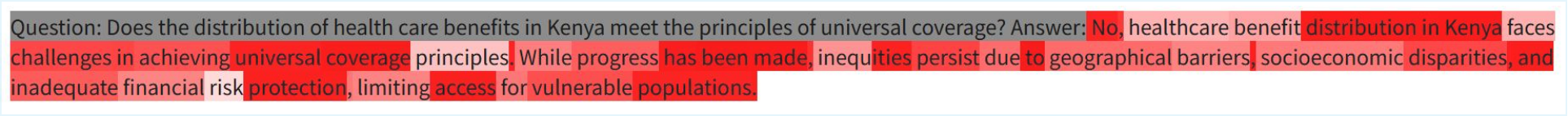}%
}

\subfloat[Human-written text, w/ \alg, Avg. log-likelihood = -3.349
]{%
    \includegraphics[width=\linewidth]{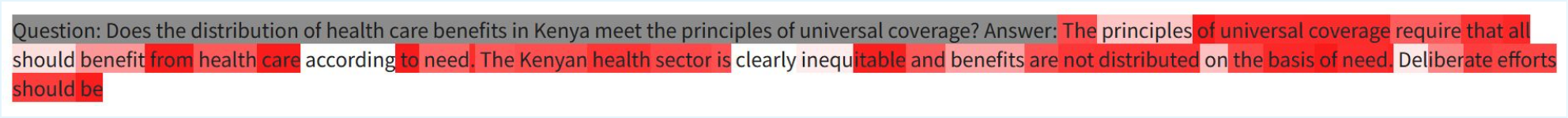}%
}

\subfloat[GPT-4-generated text, w/ \alg, Avg. log-likelihood = -2.994
]{%
    \includegraphics[width=\linewidth]{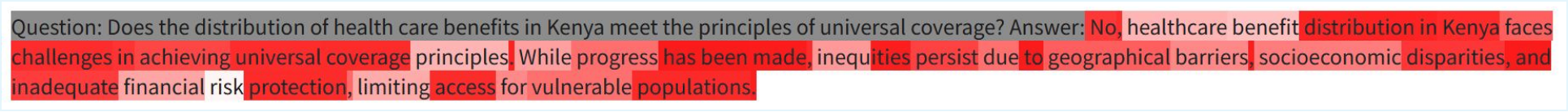}%
}

\caption{\textbf{Token-level log-likelihood visualization for a QA sample under Falcon-7B w/wo \alg.} 
Gray tokens denote the prompt, while the answer tokens are colored by their log-likelihood under the scoring model (white/red = lower/higher log-likelihood).
We compare (a) human-written text and (b) GPT-4 generated text w/o \alg, vs. (c) human-written text and (d) GPT-4 generated text w/ \alg.}
\label{fig:vis_sen}
\end{figure*}

\subsection{Visualization}
\label{sec:visualization}

We illustrate how \alg alignment changes the token-level log-likelihood landscape in Fig.~\ref{fig:vis_sen}. Without \alg, the HWT receive even higher scores than the LGT, suggesting limited separability when the proxy distribution is not well aligned to the source generator. After applying \alg, the per-token log-likelihood patterns and the overall score become more differentiated. Retrieval injects corpus-supported next-token evidence into the proxy distribution, which modifies token confidences in a context-dependent manner and yields a more stable scoring signal for black-box LGT detection.

\section{Conclusion}\label{sec:conclusion}

We present \alg, a training-free and query-efficient proxy alignment approach for black-box LGT detection that leverages \knn-LM retrieval as an adapter for a fixed proxy LLM. 
By constructing a datastore and using nearest-neighbor retrieval to induce a source-aware distribution that is blended with the proxy output, \alg aligns proxy--source LLMs without fine-tuning or intensive API queries. To further handle domain shift, we introduce \algm that routes each input to a domain-specific expert, and develop a theoretical analysis of the \alg approximation error, motivating token-wise adaptive \alg hyperparameters. Finally, extensive experiments on challenging modern benchmarks for LGT detection demonstrate that \alg can provide a stable detection signal under both topic and source LLM shifts.

\bibliographystyle{IEEEtran}
\bibliography{ref}


\end{document}